%% file: main.tex
\documentclass[10pt]{article}

\usepackage[accepted]{rlc}

\usepackage[space]{grffile}
\usepackage{algorithmic}
\usepackage{algorithm}
\usepackage{amsfonts}
\usepackage{amsmath}
\usepackage{amssymb}
\usepackage{amsthm}
\usepackage{booktabs}
\usepackage{caption}
\usepackage{graphicx}
\usepackage{hyperref}
\usepackage{mathrsfs}
\usepackage{mathtools}
\usepackage{microtype}
\usepackage{multicol}
\usepackage{multirow}
\usepackage{nicefrac}
\usepackage{relsize}
\usepackage{subcaption}
\usepackage{thm-restate}
\usepackage{thmtools}
\usepackage{tikz}
\usepackage{url}
\usepackage{wrapfig}
\usepackage{xcolor}

\usetikzlibrary{arrows.meta, positioning}
\def\Eta{H}

\def\RR{\mathbb{R}}

\def\diff#1{\operatorname{d}\!{#1}}

\DeclareMathOperator*{\argmax}{arg\,max}
\DeclareMathOperator*{\expect}{\mathlarger{\mathbb{E}}}

\definecolor{DarkBlue}{HTML}{36587E}

\declaretheoremstyle[headpunct={\normalfont.}]{theoremstyle}

\makeatletter
\newcommand{\owntag}[2][\relax]{%
  \ifx#1\relax\relax\def\owntag@name{#2}\else\def\owntag@name{#1}\fi%
  \refstepcounter{equation}\tag{#2, \theequation}%
  \expandafter\ltx@label\expandafter{eq:\owntag@name}%
  \edef\@currentlabel{#2, \theequation}\expandafter\ltx@label\expandafter{Eq:\owntag@name}%
  \def\@currentlabel{#2}\expandafter\ltx@label\expandafter{tag:\owntag@name}%
}
\makeatother

\makeatletter
\newcommand*\wt[1]{\mathpalette\wthelper{#1}}
\newcommand*\wthelper[2]{%
        \hbox{\dimen@\accentfontxheight#1%
                \accentfontxheight#11.2\dimen@
                $\m@th#1\widetilde{#2}$%
                \accentfontxheight#1\dimen@
        }%
}
\newcommand*\accentfontxheight[1]{%
        \fontdimen5\ifx#1\displaystyle
                \textfont
        \else\ifx#1\textstyle
                \textfont
        \else\ifx#1\scriptstyle
                \scriptfont
        \else
                \scriptscriptfont
        \fi\fi\fi3
}
\makeatother

\title{Informed POMDP: Leveraging Additional \\ Information in Model-Based RL}

\author{%
    Gaspard Lambrechts \\
    \texttt{gaspard.lambrechts@uliege.be} \\
    Montefiore Institute, University of Liège \\
    \And
    Adrien Bolland \\
    \texttt{adrien.bolland@uliege.be} \\
    Montefiore Institute, University of Liège \\
    \And
    Damien Ernst \\
    \texttt{dernst@uliege.be} \\
    Montefiore Institute, University of Liège \\
    LTCI, Télécom Paris, Institut Polytechnique de Paris \\
}

\begin{document}

\maketitle

\begin{abstract}
    In this work, we generalize the problem of learning through interaction in a POMDP by accounting for eventual additional information available at training time.
    First, we introduce the informed POMDP, a new learning paradigm offering a clear distinction between the information at training and the observation at execution.
    Next, we propose an objective that leverages this information for learning a sufficient statistic of the history for the optimal control.
    We then adapt this informed objective to learn a world model able to sample latent trajectories.
    Finally, we empirically show a learning speed improvement in several environments using this informed world model in the Dreamer algorithm.
    These results and the simplicity of the proposed adaptation advocate for a systematic consideration of eventual additional information when learning in a POMDP using model-based RL.
\end{abstract}

\section{Introduction} \label{sec:introduction}

Reinforcement learning (RL) aims to learn to act optimally through interaction with environments whose dynamics are unknown.
A major challenge in this field is partial observability, where only a partial observation $o$ of the Markovian state of the environment $s$ is available for taking action $a$.
Such an environment can be formalized as a partially observable Markov decision process (POMDP).
In this context, an optimal policy $\eta(a | h)$ generally depends on the history $h$ of all observations and previous actions, which grows linearly with time.
Fortunately, it is theoretically possible to find a statistic $f(h)$ of the history $h$ that is updated recurrently and that summarizes all relevant information to act optimally.
Such a statistic is said to be recurrent and sufficient for the optimal control.
Formally, a statistic $f(h)$ is recurrent when it is updated according to $f(h') = u(f(h), a, o')$ each time an action $a$ is taken and a new observation $o'$ is received, with $h' = (h, a, o')$.
And a statistic $f(h)$ is sufficient for the optimal control when there exists an optimal policy $\eta(a | h) = \allowbreak g(a | f(h))$.

In view of the existence of recurrent and sufficient statistics, many approaches have relied on learning a recurrent policy $\eta_{\theta, \phi}(a | h) = \allowbreak g_\phi(a | f_\theta(h))$ using a recurrent neural network (RNN) $f_\theta$ for the statistic.
These policies are simply trained by stochastic gradient ascent of a RL objective using backpropagation through time \citep{bakker2001reinforcement, wierstra2010recurrent, hausknecht2015deep, heess2015memory, zhang2016learning, zhu2017improving}.
In this case, the RNN learns a sufficient statistic $f_\theta(h)$ as it learns an optimal policy \citep{lambrechts2022recurrent, hennig2023emergence}.
Although these approaches theoretically allow implicit learning of a sufficient statistic, sufficient statistics can also be learned explicitly.
Notably, many works \citep{igl2018deep, buesing2018learning, guo2018neural, gregor2019shaping, han2019variational, guo2020bootstrap, lee2020stochastic, hafner2019learning, hafner2020dream} focused on learning a recurrent statistic that encodes the reward and next observation distribution given the action: $p(r, o' | h, a) = p(r, o' | f(h), a)$, a property known as predictive sufficiency \citep{bernardo2009bayesian}.
A recurrent and predictive statistic is indeed proven to be sufficient for the optimal control \citep{subramanian2022approximate}.
The sufficiency objective is usually pursued jointly with the RL objective.

While these methods can learn sufficient statistics and optimal policies in the context of POMDPs, they learn solely from the observations.
However, assuming the same partial observability at training time and execution time is too pessimistic for many environments, notably for those that are simulated.
We claim that additional information about the state $s$, be it partial or complete, can be leveraged during training for learning sufficient statistics more efficiently.
To this end, we generalize the problem of learning from interaction in a POMDP by proposing the informed POMDP.
This formalization introduces the training information $i$ about the state $s$, which is only available at training time.
Importantly, this training information is designed such that the observation is conditionally independent of the state given the information.
Note that it is always possible to design such an information $i$, possibly by concatenating the observation $o$ with the eventual additional observations $o^{\scriptscriptstyle+}$, such that $i = (o, o^{\scriptscriptstyle+})$.
This formalization offers a new learning paradigm where the training information is used along the reward and observation to supervise the learning of the policy.

In this context, we prove that recurrent statistics are sufficient for the optimal control when they are predictive sufficient for the reward and next information given the action: $p(r, i' | h, a) = p(r, i' | f(h), a)$.
We then derive a learning objective for finding a predictive sufficient statistic, which amounts to approximating the conditional distribution $p(r, i' | h, a)$ through likelihood maximization using a model $q_\theta(r, i' | f_\theta(h), a)$, where $f_\theta$ is the recurrent statistic.
Compared to the classic objective for learning sufficient statistics \citep{igl2018deep, buesing2018learning, han2019variational, hafner2019learning}, this objective approximates $p(r, i' | h, a)$ instead of $p(r, o' | h, a)$.
Next, we show that this learned model $q_\theta(r, i'| f_\theta(h), a)$ can be adapted to provide a world model from which latent trajectories can be sampled without explicitly reconstructing the observation.
This approach boils down to adapting latent world models such as those of PlaNet or Dreamer \citep{hafner2019learning, hafner2020dream, hafner2021mastering, hafner2023mastering} by relying on a model of the information instead of a model of the observation.
Our claims are supported by experiments in several environments that we formalize as informed POMDPs (Mountain Hike, Velocity Control, Pop Gym, Flickering Atari and Flickering Control).
The informed adaptation of Dreamer exhibits an improvement in terms of convergence speed and policy performance in many environments, while sometimes hurting performance in others.

This work is structured as follows.
In \autoref{sec:related}, we present some related works in asymmetric learning and multi-agent RL.
In \autoref{sec:informed}, the informed POMDP is presented with the underlying execution POMDP.
In \autoref{sec:optimal}, we provide a learning objective for sufficient statistics in this context.
In \autoref{sec:model}, we adapt the Dreamer algorithm to informed POMDPs using this informed objective.
In \autoref{sec:experiments}, we compare the Uninformed Dreamer and the Informed Dreamer in several environments.

\section{Related works} \label{sec:related}

In RL for POMDPs, asymmetric learning consists of exploiting state information during training.
These approaches usually learn policies for the POMDP by imitating a policy conditioned on the state \citep{choudhury2018data}.
However, these heuristic approaches lack a theoretical framework, and the resulting policies are known to be suboptimal for the POMDP \citep{warrington2021robust, baisero2022asymmetric}.
Intuitively, optimal policies in POMDP might indeed need to consider actions that reduce state uncertainty.
\citet{warrington2021robust} addressed this issue by constraining the expert policy so that its imitation results in an optimal policy in the POMDP.
Alternatively, asymmetric actor-critic approaches use a critic conditioned on the state \citep{pinto2017asymmetric}.
These approaches were proven to provide biased gradients by \citet{baisero2022unbiased}, who also proposed an unbiased actor-critic approach by introducing the history-state value function $V(h, s)$.
\citet{baisero2022asymmetric} adapted this method to value-based RL, where the history-dependent value function $V(h)$ uses the history-state value function $V(h, s)$ in its temporal difference target.
Alternatively, \citet{nguyen2021belief} proposed to enforce that the statistic $f(h)$ encodes the belief $p(s | h)$, a sufficient statistic for the optimal control \citep{astrom1965optimal}.
It makes the strong assumption that beliefs $b(s) = p(s|h)$ are available at training time.
Finally, in a concurrent work, \citet{avalos2024wasserstein} learns a statistic $f(h)$ that encodes the belief distribution $p(s | h)$ by leveraging the states during training.

In multi-agent RL, exploiting additional information available at training time was extensively studied under the centralized training and decentralized execution (CTDE) framework \citep{oliehoek2008optimal}.
In CTDE, it is assumed that the histories of all agents, or even the environment state, are available to all agents at training time.
To exploit this additional information, several asymmetric actor-critic approaches have been developed by leveraging an asymmetric critic conditioned on all histories, including COMA \citep{foerster2018counterfactual}, MADDPG \citep{lowe2017multi}, M3DDPG \citep{li2019robust} and R-MADDPG \citep{wang2020partially}.
While efficient in practice, \citet{lyu2022deeper} showed that these asymmetric actor-critic approaches provide biased gradient estimates, which generalizes results developed for asymmetric learning in POMDP \citep{baisero2022unbiased} to the multi-agent setting.
In the cooperative CTDE setting, another line of work focuses on value decomposition to learn a utility function for each agent, including QMIX \citep{rashid2018qmix}, QVMix \citep{leroy2021qvmix} and QPLEX \citep{wang2021qplex}.
These approaches use the additional information to modulate the contribution of each utility function in the global value function, while ensuring that maximizing the local utility functions also maximize the global value function, a property known as individual global max (IGM).
Other methods relax this IGM requirement but still condition the value function on all histories, including QTRAN \citep{son2019qtran} and WQMix \citep{rashid2020weighted}.
Recently, \citet{hong2022rethinking} established that the IGM decomposition is not attainable in the general case.

In contrast to the existing literature on asymmetric learning in POMDP, we introduce an objective that provides a sufficient statistic for the optimal control, and that leverages the additional information only through the objective.
Moreover, our new learning paradigm is not restricted to state supervision, but supports any level of additional information.
Finally, to the best of our knowledge, our method is the first to exploit additional information for learning an environment model of the POMDP.
While our approach is probably applicable to the CTDE setting for learning sufficient statistics from the local histories of each agent, we leave it as future work.

\section{Informed POMDP} \label{sec:informed}

In this section, we introduce the informed POMDP and the associated training information, along with the underlying execution POMDP and the RL objective in this context.

\subsection{Informed POMDP and execution POMDP} \label{subsec:informed}

\begin{wrapfigure}{r}{0.45\textwidth}
    \centering
    \vspace{-3.05ex}
    \begin{tikzpicture}[font=\scriptsize]
        \input{tikz/styles.tex}
        \input{tikz/ipomdp.tex}
        \input{tikz/dynamics.tex}
        \input{tikz/legend.tex}
    \end{tikzpicture}
    \vspace{-3ex}
    \caption{Bayesian network of an informed POMDP execution.}
    \label{fig:ipomdp}
    \vspace{-3ex}
\end{wrapfigure}

Formally, an informed POMDP ${\wt{\mathcal{P}}}$ is defined as a tuple ${\wt{\mathcal{P}} = (\mathcal{S}, \mathcal{A}, \mathcal{I}, \mathcal{O}, T, R, \wt{I}, \wt{O}, P, \gamma)}$ where $\mathcal{S}$ is the state space, $\mathcal{A}$ is the action space, $\mathcal{I}$ is the information space, and $\mathcal{O}$ is the observation space.
The initial state distribution $P$ gives the probability $P(s_0)$ of $s_0 \in \mathcal{S}$ being the initial state of the decision process.
The dynamics are described by the transition distribution $T$ that gives the probability $T(s_{t+1} | s_t, a_t)$ of $s_{t+1} \in \mathcal{S}$ being the state resulting from action $a_t \in \mathcal{A}$ in state $s_t \in \mathcal{S}$.
The reward function $R$ gives the expected immediate reward $r_t = R(s_t, a_t)$ obtained at each transition.
The information distribution ${\wt{I}}$ gives the probability ${\wt{I}(i_t | s_t)}$ to get information $i_t \in \mathcal{I}$ in state $s_t \in \mathcal{S}$, and
the observation distribution ${\wt{O}}$ gives the probability ${\wt{O}(o_t | i_t)}$ to get observation $o_t \in \mathcal{O}$ given information $i_t$.
Finally, the discount factor $\gamma \in \left[0, 1\right]$ gives the relative importance of future rewards.
The main assumption about an informed POMDP is that the observation $o_t$ is conditionally independent of the state $s_t$ given the information $i_t$: ${p(o_t | i_t, s_t) = \wt{O}(o_t | i_t)}$.
In other words, the random variables $s_t$, $i_t$ and $o_t$ satisfy the Bayesian network $s_t \longrightarrow i_t \longrightarrow o_t$.
In practice, it is always possible to define such a training information $i_t$.
For example, the information $i_t = (o_t, o_t^{\scriptscriptstyle+})$ satisfies the aforementioned conditional independence for any $o_t^{\scriptscriptstyle+}$.
Taking a sequence of $t$ actions in the informed POMDP conditions its execution and provides samples $(i_0, o_0, a_0, r_0, \dots, i_t, o_t)$ at training time, as illustrated in \autoref{fig:ipomdp}.

For each informed POMDP, there is an underlying execution POMDP that is defined as $\mathcal{P} = (\mathcal{S}, \mathcal{A}, \mathcal{O}, T, R, O, P, \gamma)$, where ${O(o_t | s_t) = \int_{\mathcal{I}} \wt{O}(o_t | i) \wt{I}(i | s_t) \diff i}$.
Taking a sequence of $t$ actions in the execution POMDP conditions its execution and provides the history $h_t = (o_0, a_0, \dots, o_t) \in \mathcal{H}$, where $\mathcal{H}$ is the set of histories of arbitrary length.
Note that the information samples $i_0, \dots, i_t$ and reward samples $r_0, \dots, r_{t-1}$ are not included, since they are not available at execution time.

\subsection{RL objective} \label{subsec:objective}

A policy $\eta \in \Eta$ is a mapping from histories to probability measures over the action space, where $\Eta = \mathcal{H} \rightarrow \Delta(\mathcal{A})$ is the set of such mappings.
A policy is said to be optimal for an informed POMDP when it is optimal in the underlying execution POMDP, i.e., when it maximizes the expected return
\begin{align}
    J(\eta) = \expect_{
        \substack{
            P(s_0) \\
            O(o_t | s_t) \\
            \eta(a_t | h_t) \\
            T(s_{t+1} | s_t, a_t)
        }
    } \left[
        \sum_{t=0}^\infty \gamma^t R(s_t, a_t)
    \right]\!.
    \label{eq:return}
\end{align}
The RL objective for an informed POMDP is thus to find an optimal policy $\eta^* \in \argmax_{\eta \in \Eta} J(\eta)$ for the execution POMDP from interaction with the informed POMDP.

\section{Optimal control with recurrent sufficient statistics} \label{sec:optimal}

In this section, we introduce the notion of sufficient statistic for the optimal control and derive an objective for learning such a statistic in an informed POMDP.
For the sake of conciseness, we simply use $x$ to denote a random variable at the current time step and $x'$ to denote it at the next time step.
Moreover, we use the composition notation $g \circ f$ to denote the history-dependent policy $g(\cdot | f(\cdot))$.

\subsection{Recurrent sufficient statistics} \label{subsec:recurrent}

Let us first define the concept of sufficient statistic, and derive a necessary condition for optimality.
\begin{restatable}[Sufficient statistic]{definition}{def:sufficiency} \label{defsufficiency}
    In an informed POMDP $\wt{\mathcal{P}}$ and in its underlying execution POMDP $\mathcal{P}$, a statistic of the history $f\colon \mathcal{H} \rightarrow \mathcal{Z}$ is sufficient for the optimal control if, and only if,
    \begin{align}
        \max_{g\colon \mathcal{Z} \rightarrow \Delta(\mathcal{A})} J(g \circ f) = \max_{\eta\colon \mathcal{H} \rightarrow \Delta(\mathcal{A})} J(\eta).
        \label{eq:sufficient}
    \end{align}
\end{restatable}
\begin{restatable}[Sufficiency of optimal policies]{corollary}{cor:sufficiency} \label{corsufficiency}
    In an informed POMDP $\wt{\mathcal{P}}$ and in its underlying execution POMDP $\mathcal{P}$, if a policy $\eta = g \circ f$ is optimal, then the statistic $f: \mathcal{H} \rightarrow \mathcal{Z}$ is sufficient for the optimal control.
\end{restatable}
In this work, we focus on learning recurrent policies, i.e., policies $\eta = g \circ f$ for which the statistic $f$ is recurrent.
Formally, we have,
\begin{align}
    \eta(a | h) &= g(a | f(h)), \ \forall (h, a),\label{eq:decomposition} \\
    f(h') &= u(f(h), a, o'), \ \forall h' = (h, a, o'). \label{eq:recurrent}
\end{align}
This enables the history to be processed iteratively each time that an action is taken and an observation is received.
According to \autoref{corsufficiency}, when learning a recurrent policy $\eta = g \circ f$, the objective can be broken down into two problems: finding a sufficient statistic $f$ and an optimal distribution $g$,
\begin{align}
    \max_{
        \substack{
            f\colon \mathcal{H} \rightarrow \mathcal{Z} \\
            g\colon \mathcal{Z} \rightarrow \Delta(\mathcal{A})
        }
    } J(g \circ f).
    \label{eq:objective}
\end{align}

\subsection{Learning recurrent sufficient statistics} \label{subsec:learning}

Below, we provide a sufficient condition for a statistic to be sufficient for the optimal control.
\begin{restatable}[Sufficiency of recurrent predictive sufficient statistics]{theorem}{thsufficiency} \label{th:sufficiency}
    In an informed POMDP $\wt{\mathcal{P}}$, a statistic $f\colon \mathcal{H} \rightarrow \mathcal{Z}$ is sufficient for the optimal control if it is (i)~recurrent and (ii)~predictive sufficient for the reward and next information given the action,
    \begin{align}
        \text{\normalfont(i)} \; & f(h') = u(f(h), a, o'), \ \forall h' = (h, a, o'), \label{eq:recurrent_criterion} \\
        \text{\normalfont(ii)} \; & p(r, i' | h, a) = p(r, i' | f(h), a), \  \forall (h, a, r, i'). \label{eq:sufficiency_criterion}
    \end{align}
\end{restatable}
The proof for this theorem is in \autoref{app:sufficiency}, generalizing earlier work by \citet{subramanian2022approximate}.

Now, let us consider a distribution over the histories and actions whose density function is denoted as $p(h, a)$.
For example, we consider the stationary distribution induced by the current policy $\eta$ in the informed POMDP $\wt{\mathcal{P}}$.
Let us also assume that the density function $p(h, a)$ is non-zero everywhere.
As shown in \autoref{app:objective}, under mild assumptions, any statistic $f$ satisfying the objective
\begin{align}
    & \max_{
        \substack{
            f\colon \mathcal{H} \rightarrow \mathcal{Z} \\
            q\colon \mathcal{Z} \times \mathcal{A} \rightarrow \Delta(\RR \times \mathcal{I})
        }
    } \expect_{p(h, a, r, i')} \log q(r, i' | f(h), a)
    \label{eq:variational_objective}
\end{align}
also satisfies (ii).
This variational objective jointly optimizes the statistic function $f\colon \mathcal{H} \rightarrow \mathcal{Z}$ with a conditional probability density function $q\colon \mathcal{Z} \times \mathcal{A} \rightarrow \Delta(\mathbb{R} \times \mathcal{I})$.
According to \autoref{th:sufficiency}, a statistic that is recurrent and that satisfies objective \eqref{eq:variational_objective} is sufficient for the optimal control.

In practice, both the recurrent statistic and the density function are implemented with neural networks $f_\theta$ and $q_\theta$ respectively, both parametrized by $\theta \in \RR^d$.
In this case, the objective can be maximized by stochastic gradient ascent.
Regarding the statistic function $f_\theta$, it is implicitly implemented by the update function $z_t = u_\theta(z_{t-1}; x_t)$ of an RNN.
The inputs are $x_t = (a_{t-1}, o_t)$, with $a_{-1}$ the null action that is typically set to zero.
The hidden state of the RNN $z_t = f_\theta(h_t)$ is thus a statistic of the history that is recurrently updated using $u_\theta$.
Regarding $q_\theta$, it is implemented by a parametrized probability density function estimator.
In such a context, we obtain the objective
\begin{align}
    & \max_\theta \underbrace{\expect_{p(h, a, r, i')} \log q_\theta(r, i' | f_\theta(h), a)}_{ L(f_\theta)}.
    \label{eq:neural_variational_objective}
\end{align}

We might wonder whether this informed objective is better than the classic objective, where $i = o$.
In this work, we hypothesize that approximating the information distribution instead of the observation distribution is a better objective in practice.
This is motivated by the data processing inequality applied to the Bayesian network $s' \longrightarrow i' \longrightarrow o'$, which concludes that the information $i'$ is more informative than the observation $o'$ about the Markovian state $s'$ of the environment,
\begin{align}
    I(s', i' | h, a) \geq I(s', o' | h, a), \label{eq:informativeness}
\end{align}
where $I$ denotes the conditional mutual information.
We thus expect the statistic $f_\theta(h)$ to converge faster towards a sufficient statistic, and the policy to converge faster towards an optimal policy.
It is however important to note that the information $i$ might contain irrelevant state variables.
In practice, the conditional distribution $p(i' | h, a)$ may thus be much more difficult to approximate than $p(o' | h, a)$, while not being much more useful to the control task.
While we consider this study out of the scope of this work, ensuring that the sufficient representations of the histories are also necessary for the control task is a promising avenue for future work.

\subsection{Optimal control with recurrent sufficient statistics} \label{subsec:joint}

As seen from \autoref{corsufficiency}, sufficient statistics are needed for the optimal control of POMDPs.
Moreover, as we focus on recurrent policies implemented with RNNs, we can exploit objective \eqref{eq:neural_variational_objective} to learn a sufficient statistic $f_\theta$.
In practice, we jointly maximize the RL objective $J(\eta_{\theta, \phi}) = J(g_\phi \circ f_\theta)$ and the statistic objective $L(f_\theta)$.
This enables one to use the information $i$ to guide the statistic learning through $L(f_\theta)$.
This joint maximization results in the objective
\begin{align}
    & \max_{\theta, \phi} J(g_\phi \circ f_\theta) + L(f_\theta).
    \label{eq:joint_objective}
\end{align}
Note that a policy maximizing \eqref{eq:joint_objective} also maximizes the return $J(g_\phi \circ f_\theta)$ if $f_\theta$ and $q_\theta$ are expressive enough, such that this objective provides optimal policies in the sense of objective \eqref{eq:objective}.

\section{Model-based RL through informed world models} \label{sec:model}

Model-based RL focuses on learning a model of the dynamics $p(r, o' | h, a)$ of the environment, known as a world model, that is exploited to derive a near-optimal policy.
Since the approximate model usually allows one to generate trajectories, many works derive a near-optimal policy by online planning (e.g., model-predictive control) or by optimizing a parametrized policy based on these trajectories \citep{sutton1991dyna, ha2018recurrent, chua2018deep, zhang2019solar, hafner2019learning, hafner2020dream}.
In this section, we first modify the model $q_\theta(r, i' | f_\theta(h), a)$ in order to get a world model from which trajectories can be sampled.
We then adapt the DreamerV3 \citep{hafner2023mastering} algorithm using this world model, resulting in the Informed Dreamer algorithm.

\subsection{Informed world model} \label{subsec:informed_world_model}

We implement the informed world model with a variational RNN (VRNN) as introduced by \citet{chung2015recurrent}, also known as a recurrent state-space model (RSSM) in the RL context \citep{hafner2019learning}.
It is worth noticing that such a model performs its recurrent update using a latent stochastic representation of the observation.
When generating trajectories, it also samples latent representations of the observations without explicitly reconstructing them, which we refer to as latent trajectories.
This key design choice enables the sampling of trajectories without explicitly learning the observation distribution, but the reward and information distribution only. Formally, we have,
\begin{align}
    \hspace{3cm} \hat e &\sim q_\theta^p(\cdot | z, a), \owntag[vrnn_prior]{prior}
    \\
    \hspace{3cm} \hat r &\sim q_\theta^r(\cdot | z, \hat e), \owntag[vrnn_reward]{reward decoder}
    \\
    \hspace{3cm} \hat i' &\sim q_\theta^i(\cdot | z, \hat e), \owntag[vrnn_information]{information decoder}
\intertext{where $\hat e$ is the latent variable of the VRNN when generating trajectories.
    The prior $q_\theta^p$ and the decoders $q_\theta^i$ and $q_\theta^r$ are jointly trained with the encoder,}
    \hspace{3cm} e &\sim q_\theta^e(\cdot | z, a, o'), \owntag[vrnn_posterior]{encoder}
\intertext{to maximize the likelihood of reward and next information samples.
The latent representation $e \sim q_\theta^e(\cdot | z, a, o')$ of the next observation $o'$ can be used to update the statistic to $z'$,}
    \hspace{3cm} z' &= u_\theta(z, a, e). \owntag[vrnn_recurrence]{recurrence}
\end{align}
Note that the statistic $z$ is no longer deterministically updated to $z'$ given $a$ and $o'$, instead we have $z \sim f_\theta(\cdot | h)$, which is induced by $u_\theta$ and $q_\theta^e$.
In practice, we maximize the evidence lower bound (ELBO), a variational lower bound on the likelihood of reward and next information samples given the statistic \citep{chung2015recurrent},
\begin{align}
    \expect_{\substack{p(h, a, r, i') \\ f_\theta(z | h)}}
        \log q_\theta(r, i' | z, a) \geq \expect_{\substack{p(h, a, r, i', o') \\ f_\theta(z | h)}} \bigg[
        & \expect_{q_\theta^e(e | z, a, o')} \big[
            \label{eq:vrnn_informed_objective}
            \log q_\theta^i(i' | z, e)
            + \log q_\theta^r(r | z, e) \big]
            \nonumber
            \\[-3mm]
            &
            - \operatorname{KL}\left(q_\theta^e(\cdot | z, a, o') \parallel q_\theta^p(\cdot | z, a)\right) \bigg].
\end{align}
As illustrated in \autoref{fig:training} for a trajectory sampled in the informed POMDP, the ELBO objective maximizes the conditional log-likelihood $q_\theta^r(r| z, e)$ and $q_\theta^i(i | z, e)$ of $r$ and $i'$ for a sample of the encoder $e \sim q_\theta^e(\cdot | z, a, o')$, and minimizes the KL divergence from $q_\theta^e(\cdot | z, a, o')$ to the prior distribution $q_\theta^p(\cdot | z, a)$.
Note that when $i = o$, it corresponds to Dreamer's world model and learning objective.

\begin{figure}[h]
    \centering
    \vspace{-1ex}
    \begin{tikzpicture}[font=\scriptsize]
        \input{tikz/styles.tex}
        \begin{scope}[transparency group, opacity=0.4]
            \input{tikz/ipomdp.tex}
        \end{scope}
        \input{tikz/initial.tex}
        \input{tikz/recurrence.tex}
        \def\bend{33}
        \def\pos{-4.1}
        \def\below{25}
        \input{tikz/prior.tex}
        \input{tikz/encoder.tex}
        \input{tikz/trajectory.tex}
        \input{tikz/loss.tex}
        \node[distribution, right=-0.1cm of p-1] {$q_\theta^p$} ;
    \end{tikzpicture}
    \vspace{-1ex}
    \caption{%
        VRNN loss for a given trajectory at training time. Dependence of $q_\theta^r$ and $q_\theta^i$ on $z$ is omitted.
    }
    \vspace{-0.5ex}
    \label{fig:training}
\end{figure}

As can be noticed from \autoref{eq:vrnn_informed_objective} and \autoref{fig:training}, the encoder is conditioned on the observation and not on the information.
While this is required for the encoder to be used at execution time, it certainly loosen the lower bound and limits the quality of the conditional information distribution that can be learned.
Future work may improve the quality of the information reconstruction by considering an additional information encoder, also conditioned on the statistic of the history, whose samples are not used in the recurrence.

\subsection{Informed Dreamer} \label{subsec:informed_dreamer}

As explained above, while our informed world model does not learn the observation distribution, it is still able to sample latent trajectories.
Indeed, the VRNN only uses the latent representation $e \sim q_\theta^e(\cdot | z, a, o')$ of the observation $o'$, trained to reconstruct the information $i'$, in order to update $z$ to $z'$.
Consequently, we can use the prior distribution $\hat e \sim q_\theta^p(\cdot | z, a)$, trained according to \eqref{eq:vrnn_informed_objective} to minimize the KL divergence from $e \sim q_\theta^p(\cdot | z, a, o')$ in expectation, to sample latent trajectories.

The Informed Dreamer algorithm leverages such trajectories to learn a latent critic $v_\psi(z)$ and a latent policy $a \sim g_\phi(\cdot | z)$.
\autoref{fig:imagination} illustrates the generation of a latent trajectory, along with estimated rewards $\hat r \sim q_\theta^r(\cdot | z, e)$ and values $\hat v = v_\psi(z)$.
The actions are sampled according to the latent policy, and any RL algorithm can be used to maximize the estimated return.
Moreover, note that the estimated return is given by a function that is differentiable with respect to $\phi$, and it can be directly maximized by stochastic gradient ascent.
In the experiments, we use an actor-critic approach for discrete actions and direct maximization for continuous actions, following DreamerV3 \citep{hafner2023mastering}.
Finally, as shown in \autoref{fig:execution}, when deployed in the execution POMDP, the encoder $q_\theta^e$ is used to compute the latent representations of the observations and to update the statistic.
The actions are then selected according to $a \sim g_\phi(\cdot | z)$.

\begin{figure}[h]

    \centering
    \vspace{-1ex}
    \begin{subfigure}[b]{0.49\linewidth}
        \begin{tikzpicture}[font=\scriptsize]
            \input{tikz/styles.tex}
            \input{tikz/initial.tex}
            \input{tikz/recurrence.tex}
            \def\bend{0}
            \def\pos{-3.1}
            \def\below{25}
            \input{tikz/prior.tex}
            \input{tikz/policy.tex}
            \input{tikz/imagined.tex}
            \node[distribution, below right=-0.27cm of p-1] {$q_\theta^p$} ;
        \end{tikzpicture}
        \vspace{-0.5ex}
        \caption{Imagination of a trajectory using policy $g_\phi$ with estimated rewards and values. Dependence of $q_\theta^r$ and $v_\psi$ on $z$ is omitted.}
        \label{fig:imagination}
    \end{subfigure}
    \begin{subfigure}[b]{0.49\linewidth}
        \begin{tikzpicture}[font=\scriptsize]
            \input{tikz/styles.tex}
            \input{tikz/initial.tex}
            \input{tikz/recurrence.tex}
            \input{tikz/encoder.tex}
            \input{tikz/history.tex}
            \input{tikz/policy.tex}
        \end{tikzpicture}
        \vspace{-0.5ex}
        \caption{Execution of the policy on a trajectory of the POMDP using the encoder $q_\theta^e$ to condition the latent policy $g_\phi$.}
        \label{fig:execution}
    \end{subfigure}
    \vspace{-1ex}
    \caption{%
        Bayesian graph of a VRNN evaluation during imagination and execution.
    }
    \vspace{-0.5ex}
    \label{fig:imagination_and_execution}
\end{figure}

A pseudocode for the adaptation of the DreamerV3 algorithm using this informed world model is given in \autoref{app:dreamer}.
We also detail some divergences of our formalization with respect to the original DreamerV3 algorithm.
As in DreamerV3, we use symlog predictions, a discrete VAE, KL balancing, free bits, reward normalisation, a distributional critic, and entropy regularization.

\section{Experiments} \label{sec:experiments}

In this section, we compare Dreamer to the Informed Dreamer on several informed POMDPs, all considered with a discount factor of $\gamma = 0.997$.
For reproducibility purposes, we use the implementation and hyperparameters of DreamerV3 released by the authors at \href{https://github.com/danijar/dreamerv3}{github.com/danijar/dreamerv3}, and release our adaptation to informed POMDPs using the same hyperparameters at \href{https://github.com/glambrechts/informed-dreamer}{github.com/glambrechts/informed-dreamer}.

\subsection{Varying Mountain Hike} \label{subsec:hike}

In the Varying Mountain Hike environments, the agent should walk throughout a mountainous terrain to reach the mountain top as fast as possible while avoiding the valleys.
There exists four versions of this environment, depending on the agent orientation (\emph{north} or \emph{random}) and on the observation that is available (\emph{position} or \emph{altitude}).
More formally, the agent has a position $x$ and a fixed orientation $c$ in each episode.
The orientation $c$ is either always \emph{north} or a \emph{random} cardinal orientation, depending on the environment version.
It can take four actions to move relative to its orientation (right, forward, left and backward).
The orientation is not observed by the agent, but it receives a Gaussian observation of its \emph{position}, or its \emph{altitude}, depending on the environment version ($\sigma_o = 0.1$ in both cases).
The reward is given by its altitude relative to the mountain top, such that the goal of the agent is to obtain the highest cumulative altitude.
Around the mountain top, states are terminal and the trajectories are truncated at $t = 160$ in practice.
We refer the reader to \citet{lambrechts2022recurrent} for a formal description of these environments, strongly inspired by the Mountain Hike of \citet{igl2018deep}.

For this environment, we first consider the position and orientation to be available as additional information at training time.
In other words, we consider the state-informed POMDP with $i = s$.
As can be seen in \autoref{fig:hike}, the speed of convergence of the policies is improved in all four environments when using the Informed Dreamer.
Moreover, as shown in \autoref{tab:hike} in \autoref{app:table}, the final performance of the Informed Dreamer is better in 3 out of 4 environments.

\begin{figure}[ht]
    \centering
    \vspace{-1ex}
    \begin{subfigure}[b]{0.40\linewidth}
        \includegraphics[width=\linewidth]{{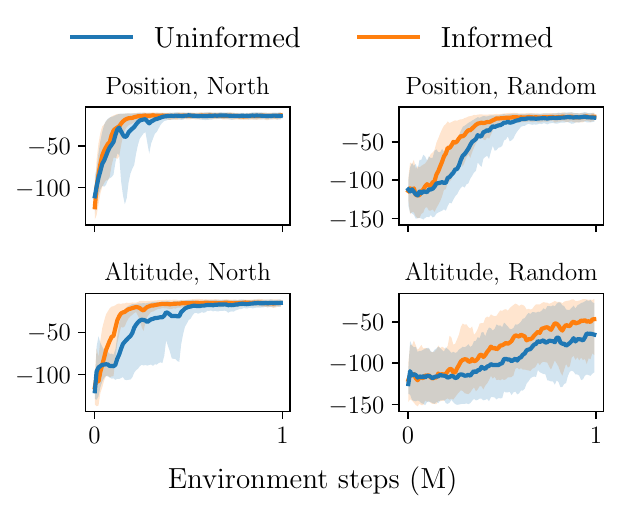}}
        \vspace{-3ex}
        \caption{%
            \scriptsize Uninformed Dreamer and Informed Dreamer with $i = s$ in the four environments.
        }
        \label{fig:hike}
    \end{subfigure}
    \hspace{0.02\linewidth}
    \begin{subfigure}[b]{0.40\linewidth}
        \includegraphics[width=\linewidth]{{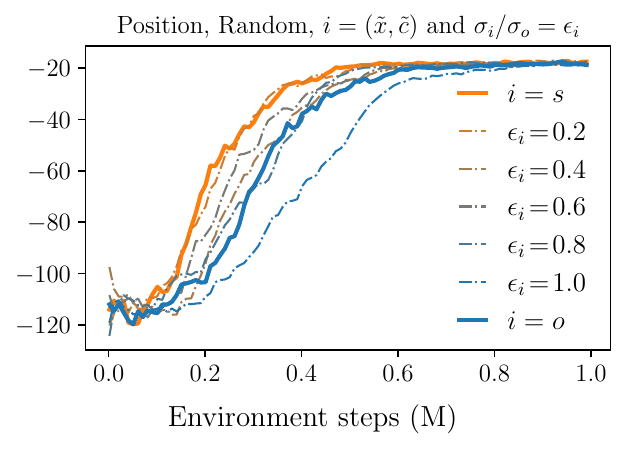}}
        \vspace{-3ex}
        \caption{%
            \scriptsize Informed Dreamer with $i = (\tilde x, \tilde c)$ with position observation and random orientation.
        }
        \label{fig:quality}
    \end{subfigure}
    \vspace{-1ex}
    \caption{%
        Varying Mountain Hike environments: minimum, maximum and average returns over five trainings.
    }
    \vspace{-0.5ex}
\end{figure}

We also experiment with other types of information in the Varying Mountain Hike with position observation and random orientation.
More precisely, we consider an information $i = (\tilde x, \tilde c)$ about the state $s = (x, c)$, where $\tilde x$ is an observation of the position $x$ with Gaussian noise of standard deviation $\sigma_i \in [0, \sigma_o]$, and $\tilde c$ is a noisy observation of the orientation $c$ replaced by a random orientation with probability $\epsilon_i \in [0, 1]$.
Note that when $\sigma_i = 0$, the position $x$ is encoded in the information, while when $\sigma_i = \sigma_o$, the observation $o$ is encoded in the information.
As shown in \autoref{fig:quality}, without confidence intervals for the sake readability, the better the information, the faster the policy converges.
It supports the idea that the more information about the state is exploited, the faster an optimal policy for the POMDP is learned.
Moreover, we observe that the Informed Dreamer with $\epsilon_i = 1$ and $\sigma_i = 0.1$ performs even worse than the Uninformed Dreamer.
It suggests that considering additional information that is not informative about the state (i.e., $I(s, i | o) = 0$), such as $\tilde c$ with $\epsilon_i = 1$, can degrade learning.
Similar results are obtained for the other three environments in \autoref{subsec:qualities}.

\subsection{Velocity Control} \label{subsec:velocity}

In the Velocity Control environments, we consider the standard DeepMind Control tasks \citep{tassa2018deepmind}, where only the joints velocities are available as observations and not their absolute positions, which is a standard benchmark in the partially observable RL literature \citep{han2019variational, lee2020stochastic}.
These environments consists of controlling different multi-joints robots to achieve several tasks.
We consider the absolute positions to be available at training time along with the velocities, which results in a Markovian information $i = s$.

\begin{figure}[ht]
    \centering
    \vspace{-1ex}
    \includegraphics[width=\linewidth]{{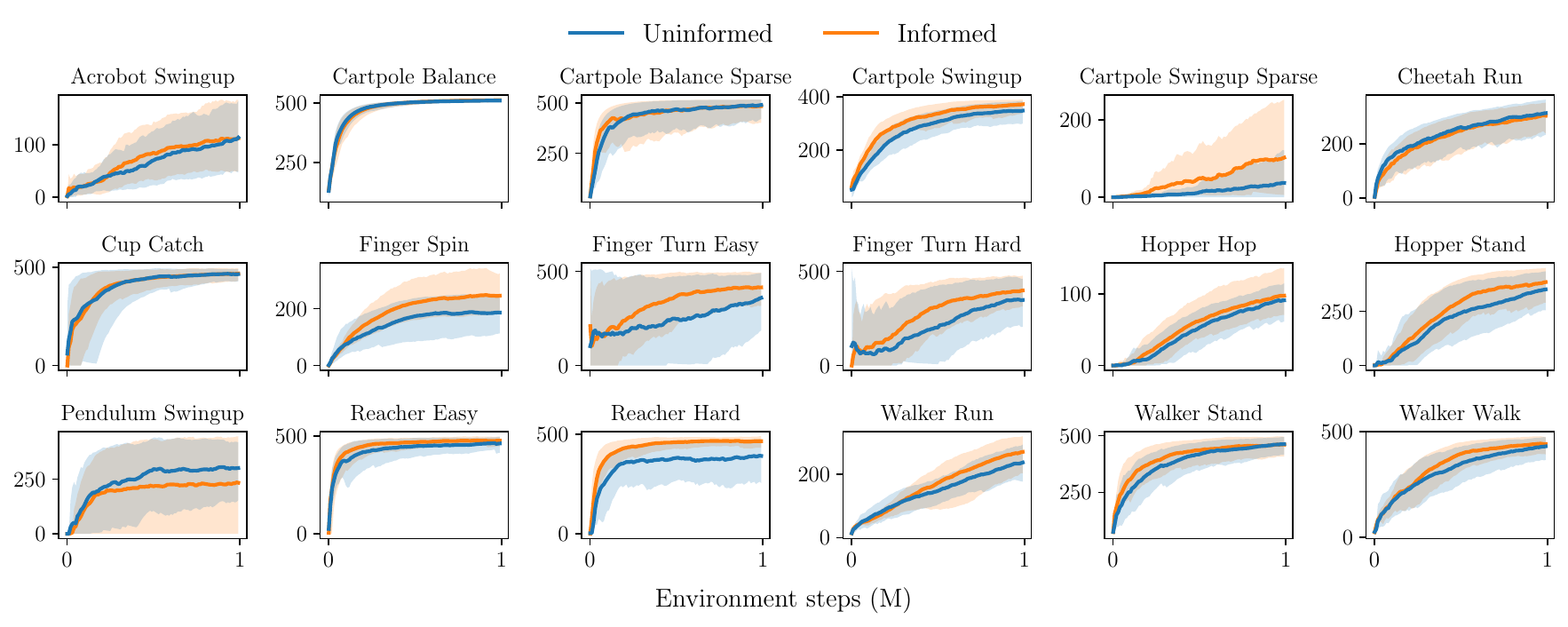}}
    \vspace{-4ex}
    \caption{
        Uninformed Dreamer and Informed Dreamer with $i = s$ in the Velocity Control environments: minimum, maximum and average returns over five trainings.
    }
    \label{fig:velocity}
    \vspace{-0.5ex}
\end{figure}

\autoref{fig:velocity} shows that the convergence speed of the policies is improved in this benchmark, for nearly all of the considered games.
Moreover, the final returns are given in \autoref{tab:velocity} in \autoref{app:table}, and show that policies obtained after one million time steps are better in 13 out of 18 environments when considering additional information.

\subsection{Pop Gym} \label{subsec:pop}

The Pop Gym environments have been specifically designed to benchmark the ability of handling partial observability \citep{morad2023popgym}.
The latter notably includes memory games, board games, or control problems involving partial observability and noise.
For these environments, we consider the state to be available as additional information.

\begin{figure}[ht]
    \centering
    \vspace{-1ex}
    \includegraphics[width=0.86\linewidth]{{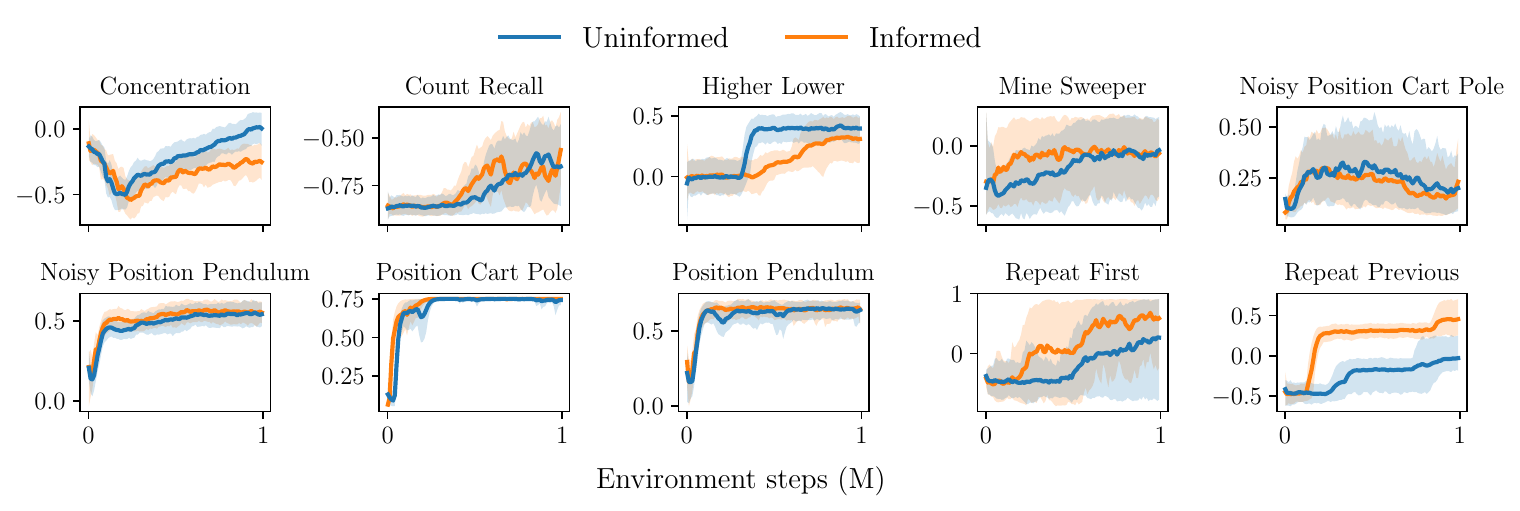}}
    \vspace{-2ex}
    \caption{
        Uninformed Dreamer and Informed Dreamer with $i = s$ in the Pop Gym environments: minimum, maximum and average returns over five trainings.
    }
    \label{fig:pop}
    \vspace{-0.5ex}
\end{figure}

\autoref{fig:pop} shows that learning in those POMDPs sometimes benefits from the exploitation of additional information as proposed in the Informed Dreamer.
The learning of the Informed Dreamer seems to suffer from the approximation of the information distribution in 2 out of those 10 environments (Concentration and Higher Lower).
The final returns are given in \autoref{tab:pop} in \autoref{app:table}, showing a better final performance in 7 out of 10 environments, even though returns have a high variability.
In particular, we observe that the Informed Dreamer converges to a significantly higher return for the Repeat First and Repeat Previous environments, that both require discovering long time dependencies.
The exploitation of additional information seems crucial in these environments, and we study this in depth on harder instances of the Repeat Previous environment in \autoref{subsec:hard}.
This analysis shows that the Informed Dreamer can learn near-optimal policies in environments for which the Uninformed Dreamer does not learn at all.

\subsection{Flickering Atari and Flickering Control}

While arguably not constituting a relevant benchmark for measuring the ability of handling partial observability \citep{shao2022mask, avalos2024wasserstein}, the Flickering Atari and Flickering Control environments have become standard benchmarks in the partially observable RL literature \citep{hausknecht2015deep, zhu2017improving, igl2018deep, ma2020discriminative}.
For completeness, the results for these environments are reported in \autoref{app:flickering}.
We observe that the speed of convergence and final performance of the agent is sometimes greatly improved when considering additional information (e.g., Asteroids, Pong, Breakout).
However, we also observe that the performance is lower in some environments.
As far as the Flickering Atari environments are concerned, the Informed Dreamer only outperforms Dreamer in 6 out of 12 environments.
In the Flickering Control environments, the Informed Dreamer tends to systematically underperform the Uninformed Dreamer, attaining a better performance in only 2 out of 18 environments.
It suggests that additional state information is not useful for these tasks.
We furthermore hypothesize that the conditional information distribution is difficult to approximate, which may cause learning to degrade.
It shows that not all information is worth exploiting, particularly when the level of uncertainty due to partial observability is low.

\section{Conclusion} \label{sec:conclusion}

In this work, we introduced a new formalization for considering additional information available at training time for POMDP, called the informed POMDP.
In this context, we proposed a learning objective and proved that it provides sufficient statistic for the optimal control.
Next, we adapted this objective to provide an environment model from which latent trajectories can be sampled.
We then adapted a successful model-based RL algorithm, known as Dreamer, with this informed world model, resulting in the Informed Dreamer algorithm.
By considering several environments from the partially observable RL literature, we showed that this informed learning objective often improves the convergence speed and quality of the policies.
This work also presents several limitations.
First, a formal justification for the use of the information instead of the observation is still lacking.
Future work may consider the notion of approximate information states to bound the suboptimality of the policy for a given error on the information distribution instead of the observation distribution.
Second, we observed that this informed objective hurts performance in some environments, motivating further work in which particular attention is paid to the design of the information.
It would be worth drawing connection to the exogenous RL literature that complements this work by focusing on discarding irrelevant information.
Third, the proposed ELBO learning objective is probably a loose lower bound on the information likelihood. Future work may improve the quality of the information distribution by considering informed world models with a dedicated information encoder.

\subsubsection*{Acknowledgements}

The authors would like to thank our colleagues Pascal Leroy, Arnaud Delaunoy, Renaud Vandeghen and Florent De Geeter for their valuable comments on this manuscript.
Gaspard Lambrechts gratefully acknowledges the financial support of the \emph{Wallonia-Brussels Federation} for his FRIA grant.
Adrien Bolland gratefully acknowledges the financial support of the \emph{Wallonia-Brussels Federation} for his FNRS grant.
Computational resources have been provided by the \emph{Consortium des Équipements de Calcul Intensif} (CÉCI), funded by the \emph{National Fund for Scientific Research} (F.R.S.-FNRS) under Grant No. 2502011 and by the \emph{Walloon Region}, including the Tier-1 supercomputer of the \emph{Wallonia-Brussels Federation}, infrastructure funded by the \emph{Walloon Region} under Grant No. 1117545.

\bibliographystyle{rlc}
\bibliography{references.bib}

\newpage
\appendix
\onecolumn

\section{Sufficiency of recurrent predictive sufficient statistics} \label{app:sufficiency}

In this section, we prove \autoref{th:sufficiency}, that is recalled below.
\thsufficiency*
\begin{proof}
    From Proposition~4 and Theorem~5 by \citet{subramanian2022approximate}, we know that a statistic is sufficient for the optimal control of an execution POMDP if it is (i) recurrent and (ii') predictive sufficient for the reward and next \emph{observation} given the action: $p(r, o' | h, a) = p(r, o' | f(h), a)$.
    Let us consider a statistic $f\colon \mathcal{H} \rightarrow \mathcal{A}$ satisfying (i) and (ii).
    Let us show that it satisfies (ii'). We have,
    \begin{align}
        p(r, o' | f(h), a)
        &= \int_{\mathcal{I}} p(r, o', i' | f(h), a) \diff i' \label{eq:suff1} \\
        &= \int_{\mathcal{I}} p(o' | r, i', f(h), a) p(r, i' | f(h), a) \diff i', \label{eq:suff2}
    \end{align}
    using the law of total probability and the chain rule.
    As can be seen from the informed POMDP formalization of \autoref{sec:informed} and the resulting Bayesian network in \autoref{fig:ipomdp}, the Markov blanket of $o'$ is $\left\{i'\right\}$.
    As a consequence, $o'$ is conditionally independent of any other variable given $i'$.
    In particular, $p(o' | i', r, f(h), a) = p(o | i')$, such that,
    \begin{align}
        p(r, o' | f(h), a)
        &= \int_{\mathcal{I}} p(o' | i') p(r, i' | f(h), a) \diff i'. \label{eq:suff3}
    \end{align}
    From hypothesis (ii), we can write,
    \begin{align}
        p(r, o' | f(h), a)
        &= \int_{\mathcal{I}} p(o' | i') p(r, i' | h, a) \diff i'. \label{eq:suff4}
    \end{align}
    Finally, exploiting the Markov blanket $\left\{i'\right\}$ of $o'$, the chain rule and the law of total probability again, we have,
    \begin{align}
        p(r, o' | f(h), a)
        &= \int_{\mathcal{I}} p(o' | i', r, h, a) p(r, i' | h, a) \diff i' \label{eq:suff5} \\
        &= \int_{\mathcal{I}} p(o', r, i' | h, a) \diff i' \label{eq:suff6} \\
        &= p(r, o' | h, a). \label{eq:end}
    \end{align}
    This proves that (ii) implies (ii').
    As a consequence, any statistic $f$ satisfying (i) and (ii) is a sufficient statistic of the history for the optimal control of the informed POMDP.
\end{proof}

\section{Recurrent sufficient statistic objective} \label{app:objective}

First, let us consider a fixed history $h$ and action $a$. Let us recall that two density functions $p(r, i' | h, a)$ and $p(r, i' | f(h), a)$ are equal almost everywhere if, and only if, their KL divergence is zero,
\begin{align}
    \expect_{p(r, i' | h, a)} \log \frac{p(r, i' | h, a)}{p(r, i' | f(h), a)} = 0.
\end{align}
Now, let us consider a probability density function $p(h, a)$ that is non zero everywhere.
We have that the KL divergence from $p(r, i' | h, a)$ to $p(r, i' | f(h), a)$ is equal to zero for almost every history $h$ and action $a$ if, and only if, it is zero on expectation over $p(h, a)$ since the KL divergence is non-negative,
\begin{align}
    \expect_{p(r, i' | h, a)} \log \frac{p(r, i' | h, a)}{p(r, i' | f(h), a)} \stackrel{\mathclap{\normalfont\mbox{\tiny\phantom{.\@}\!a.e.}}}{=} 0
    \Leftrightarrow &\expect_{p(h, a, r, i')} \log \frac{p(r, i' | h, a)}{p(r, i' | f(h), a)} = 0.
\end{align}
Rearranging, we have that $p(r, i' | h, a)$ is equal to $p(r, i' | f(h), a)$ for almost every $h$, $a$, $r$ and $i'$ if, and only if,
\begin{align}
    \expect_{p(h, a, r, i')} \log p(r, i' | h, a) = \expect_{p(h, a, r, i')} \log p(r, i' | f(h), a).
\end{align}
Now, we recall the data processing inequality, enabling one to write, for any statistic $f'$,
\begin{align}
    \expect_{p(h, a, r, i')} \log p(r, i' | h, a) \geq \expect_{p(h, a, r, i')} \log p(r, i' | f'(h), a).
\end{align}
since $h(r, i' | h, a) = h(r, i' | h, f(h), a) \leq h(r, i' | f(h), a), \ \forall (h, a)$, where $h(x)$ is the differential entropy of random variable $x$.
Assuming that there exists at least one $f\colon \mathcal{H} \rightarrow \mathcal{Z}$ for which the inequality is tight, we obtain the following objective for a predictive sufficient statistic $f$,
\begin{align}
    \max_{f\colon \mathcal{H} \rightarrow \mathcal{Z}} \expect_{p(h, a, r, i')} \log p(r, i' | f(h), a).
    \label{eq:sufficiency_criterion_log_likelihood}
\end{align}
Unfortunately, the probability density $p(r, i' | f(h), a)$ is unknown.
However, knowing that the distribution that maximizes the log-likelihood of samples from $p(r, i' | f(h), a)$ is $p(r, i' | f(h), a)$ itself, we can write,
\begin{align}
    \expect_{p(h, a, r, i')} \log p(r, i' | f(h), a)
    =
    \max_{q\colon \mathcal{Z} \times \mathcal{A} \rightarrow \Delta(\RR \times \mathcal{I})} \expect_{p(h, a, r, i')} \log q(r, i' | f(h), a)
    .
\end{align}
By jointly maximizing the probability density function $q\colon \mathcal{Z} \times \mathcal{A} \rightarrow \Delta(\RR \times \mathcal{I})$, we obtain,
\begin{align}
    \max_{
        \substack{
            f\colon \mathcal{H} \rightarrow \mathcal{Z} \\
            q\colon \mathcal{Z} \times \mathcal{A} \rightarrow \Delta(\RR \times \mathcal{I})
        }
    } \expect_{p(h, a, r, i')} \log q(r, i' | f(h), a).
\end{align}
This objective ensures that the statistic $f(h)$ is predictive sufficient for the reward and next information given the action.
If $f(h)$ is a recurrent statistic, then it is also sufficient for the optimal control, according to \autoref{th:sufficiency}.

\section{Informed Dreamer} \label{app:dreamer}

The Informed Dreamer algorithm is presented in \autoref{algo:dreamer}.
Differences with the Uninformed Dreamer algorithm \citep{hafner2020dream} are highlighted in blue.
In addition, it can be noted that in the original Dreamer algorithm, the statistic $z_{t}$ encodes $h_t = (o_0, a_0, \dots, o_t)$ and $a_t$, instead of $h_t$ only.
As a consequence, the prior distribution $e_{t} \sim q_\theta^p(\cdot | z_{t})$ can be conditioned on the statistic $z_{t}$ only, instead of the statistic and last action.
Similarly, the encoder distribution $e_{t} \sim q_\theta^p(\cdot | z_t, o_{t+1})$ can be conditioned on the statistic $z_t$ only, instead of the statistic and last action.
On the other hand, the latent policy $a_{t+1} \sim g(\cdot | z_t, e_{t})$ should be conditioned on the statistic $z_t$ and the new latent $e_t$ to account for the last observation, and the same is true for the value function $v_\psi(z_t, e_t)$.
In the experiments, we follow the original implementation for both the Uninformed Dreamer and the Informed Dreamer, according to the code that we release at \href{https://github.com/glambrechts/informed-dreamer}{github.com/glambrechts/informed-dreamer}.

Following Dreamer, the algorithm introduces the continuation flag $c_t$, which indicates whether state $s_t$ is terminal.
A terminal state $s_t$ is a state from which the agent can never escape, and in which any further action provides a zero reward.
It follows that the value function of a terminal state is zero, and trajectories can be truncated at terminal states since we do not need to learn their value or the optimal policy in those states.
Alternatively, $c_t$ can be interpreted as an indicator that can be extracted from the observation $o_t$, but we made it explicit in the algorithm.

\begin{algorithm}[p]
    \small
    \caption{Informed Dreamer - direct reward maximization}
    \label{algo:dreamer}
    \begin{algorithmic}
        \STATE {\bfseries Hyperparameters:} Environment steps $S$, steps before training $F$, train ratio $R$, backpropagation horizon $W$, imagination horizon $K$, batch size $N$, replay buffer capacity $B$.
        \vspace{.6ex}
        \STATE Initialize neural network parameters $\theta$, $\phi$, $\psi$ randomly, initialize empty replay buffer $\mathcal{B}$.
        \STATE Let $g = 0$, $t = 0$, $a_{-1} = 0$, $r_{-1} = 0$, $z_{-1} = 0$.
        \STATE Reset the environment and observe $o_0$ and $c_0$ (true at reset).
        \FOR{$s = 0 \dots S-1$}
            \vspace{.6ex}
            \STATE \textcolor{DarkBlue}{// \textit{Environment interaction}}
            \STATE Encode observation $o_t$ to $e_{t-1} \sim q_\theta^e(\cdot | z_{t-1}, a_{t-1}, o_t)$.
            \STATE Update $z_t = u_\theta(z_{t-1}, a_{t-1}, e_{t-1})$.
            \STATE Given the current statistic $z_t$, take action $a_t \sim g_\phi(\cdot | z_t)$.
            \STATE Observe reward $r_t$, \textcolor{DarkBlue}{information} $\color{DarkBlue}i_{t+1}$, observation $o_{t+1}$ and continuation flag $c_{t+1}$.
            \IF{$c_{t+1}$ is false (terminal state)}
                \STATE Reset $t = 0$.
                \STATE Reset the environment and observe $o_0$ and $c_0$ (true at reset).
            \ENDIF
            \STATE Update $t = t + 1$.
            \STATE Add trajectory of last $W$ time steps $(a_{w-1}, r_{w-1}, {\color{DarkBlue}i_w}, o_w, c_w)_{w=t-W+1}^{t}$ to replay buffer $\mathcal{B}$.
            \vspace{.6ex}
            \STATE \textcolor{DarkBlue}{// \textit{Learning}}
            \WHILE{$|\mathcal{B}| \geq F \land g < Rs$}
                \vspace{.6ex}
                \STATE \textcolor{DarkBlue}{// \textit{Environment learning}}
                \STATE Draw $N$ trajectories of length $W$ \smash{$\left\{ (a^n_{w-1}, r^n_{w-1}, {\color{DarkBlue} i^n_w}, o^n_w, c^n_w)_{w = 0}^{W-1} \right\}_{n = 0}^{N - 1}$} uniformly from replay buffer $\mathcal{B}$.
                \STATE Compute statistics and encoded latents
                $$\left\{ (z^n_w, e^n_w)_{w = -1}^{W - 2} \right\}_{n = 0}^{N - 1} = \text{\hyperlink{caption:encode}{Encode}}\left(u_\theta, q_\theta^e, \left\{ (a^n_{w-1}, o^n_w)_{w = 0}^{W-1} \right\}_{n = 0}^{N - 1}\right).$$
                \STATE Update $\theta$ using $\nabla_\theta \sum_{n=0}^{N} \sum_{w=-1}^{W-2} L_w^n$, where $a^n_{-1} = 0$ and,
                    \vspace{-1.5ex}
                    \begin{align*}
                        L_w^n = & \
                            {\color{DarkBlue} \log q_\theta^i(i^n_{w+1} | z^n_{w}, e^n_{w})}
                            + \log q_\theta^c(c^n_{w+1} | z^n_{w}, e^n_{w})
                            + \log q_\theta^r(r^n_{w} | z^n_{w}, e^n_{w}) \\
                            & - \operatorname{KL}\left(q_\theta^e(\cdot | z^n_{w}, a^n_{w}, o^n_{w+1}) \parallel q_\theta^p(\cdot | z^n_{w}, a^n_{w})\right).
                    \end{align*}
                \vspace{-2.5ex}
                \STATE \textcolor{DarkBlue}{// \textit{Behaviour learning}}
                \STATE Sample latent trajectories
                $$\left\{ \left\{ (z^{n,w}_{k}, \hat e^{n,w}_{k})_{k = 0}^{K - 1} \right\}_{w = -1}^{W-2} \right\}_{n = 0}^{N - 1} = \text{\hyperlink{caption:imagine}{Imagine}}\left(u_\theta, q^p_\theta, g_\phi, \left\{ (z^n_w, e^n_w, a^n_w)_{w = - 1}^{W-2} \right\}_{n = 0}^{N - 1}\right).$$
                \STATE Predict rewards $r^{n,w}_{k} \sim q_\theta^r(\cdot | z^{n,w}_{k}, \hat e^{n,w}_{k})$, continuations flags $c^{n,w}_{k+1} \sim q_\theta^c(\cdot | z^{n,w}_{k}, \hat e^{n,w}_{k})$, and values $v^{n,w}_{k} = v_\psi(z^{n,w}_{k})$.
                \STATE Compute value targets using $\lambda$-returns, with $G^{n,w}_{K-1} = v_{K-1}^{n,w}$ and
                \vspace{-1.5ex}
                \begin{align*}
                    G^{n,w}_{k} = r^{n,w}_{k} + \gamma c^{n,w}_{k} \left( (1 - \lambda) v_{k+1}^{n,w} + \lambda G^{n,w}_{k+1} \right).
                \end{align*}
                \vspace{-2.5ex}
                \STATE Update $\phi$ using $\nabla_\phi \sum_{n=0}^{N-1} \sum_{w = -1}^{W - 2} \sum_{k = 0}^{K-1} G^{n,w}_{k}$.
                \STATE Update $\psi$ using $\nabla_\psi \sum_{n=0}^{N-1} \sum_{w = -1}^{W - 2} \sum_{k = 0}^{K-1} \lVert v_\psi(z^{n,w}_{k}) - \operatorname{sg}(G^{n,w}_{k}) \rVert^2$, where $\operatorname{sg}$ is the stop-gradient operator.
                \STATE Count gradient steps $g = g + 1$
            \ENDWHILE
        \ENDFOR
    \end{algorithmic}
\end{algorithm}

\begin{algorithm}[!ht]
    \hypertarget{caption:encode}{}
    \small
    \caption{Encode}
    \label{algo:encode}
    \begin{algorithmic}
        \STATE {\bfseries Inputs:} Update function $u_\theta$, encoder $q_\theta^e$, and histories $\left\{ (a^n_{w-1}, o^n_w)_{w = 0}^{W-1} \right\}_{n = 0}^{N - 1}$.
        \STATE Let $z^n_{-1} = 0$.
        \FOR{$w = 0 \dots W-1$}
            \STATE Let $e^n_{w-1} \sim q_\theta^e(\cdot | z^n_{w-1}, a^n_{w-1}, o^n_{w})$.
            \STATE Let $z^n_w = u_\theta(z^n_{w-1}, a^n_{w-1}, e^n_{w-1})$.
        \ENDFOR
        \STATE {\bfseries Returns:} $\left\{ (z^n_w, e^n_w)_{w = -1}^{W - 2} \right\}_{n = 0}^{N - 1}$.
    \end{algorithmic}
\end{algorithm}

\begin{algorithm}[!ht]
    \hypertarget{caption:imagine}{}
    \small
    \caption{Imagine}
    \label{algo:imagine}
    \begin{algorithmic}
        \STATE {\bfseries Inputs:} Update function $u_\theta$, prior $q_\theta^p$, policy $g_\phi$, statistics, encoded latents and actions $\left\{ (z^n_w, e^n_w, a^n_w)_{w = -1}^{W-2} \right\}_{n = 0}^{N - 1}$.
        \STATE Let $z^{n,w}_{-1} = z^n_w$, $\hat e^{n,w}_{-1} = e^n_w$, $a^{n,w}_{-1} = a^n_w$.
        \FOR{$k = 0 \dots K-1$}
            \STATE Let $z^{n,w}_{k} = u_\theta(z^{n,w}_{k-1}, a^{n,w}_{k-1}, \hat e^{n,w}_{k-1})$.
            \STATE Let $\hat e^{n,w}_{k} \sim q_\theta^p(\cdot | z^{n,w}_{k}, a^{n,w}_{k})$.
            \STATE Let $a^{n,w}_{k} \sim g_\phi(\cdot | z^{n,w}_{k})$.
        \ENDFOR
        \STATE {\bfseries Returns:} $\left\{ \left\{ (z^{n,w}_{k}, \hat e^{n,w}_{k})_{k = 0}^{K - 1} \right\}_{w = -1}^{W-2} \right\}_{n = 0}^{N - 1}$.
    \end{algorithmic}
\end{algorithm}

\section{Final returns} \label{app:table}

We provide the final returns obtained by Dreamer and the Informed Dreamer for the Varying Mountain Hike environments in \autoref{tab:hike}, for the Velocity Control environments in \autoref{tab:velocity}, and for the Pop Gym environments in \autoref{tab:pop}.

\begin{table}[!ht]
    \fontsize{9pt}{10pt}\selectfont
    \vspace{-0.5ex}
    \caption{%
        Average final return and standard deviation over five trainings in the Mountain Hike environments.
    }
    \vspace{-1.5ex}
    \label{tab:hike}
    \begin{center}
        \input{tex/hike.tex}
    \end{center}
    \vspace{-2ex}
\end{table}

\begin{table}[!ht]
    \fontsize{9pt}{10pt}\selectfont
    \vspace{-0.5ex}
    \caption{
        Average final return and standard deviation over five trainings in the Velocity Control environments.
    }
    \vspace{-1.5ex}
    \label{tab:velocity}
    \begin{center}
        \input{tex/velocity.tex}
    \end{center}
    \vspace{-2ex}
\end{table}

\begin{table}[!ht]
    \fontsize{9pt}{10pt}\selectfont
    \vspace{-0.5ex}
    \caption{
        Average final return and standard deviation over five trainings in the Pop Gym environments.
    }
    \vspace{-1.5ex}
    \label{tab:pop}
    \begin{center}
        \input{tex/pop.tex}
    \end{center}
    \vspace{-2ex}
\end{table}

\section{Additional experiments} \label{app:flickering}

In this section, we provide results for non-Markovian information in the Varying Mountain Hike environments, for harder Pop Gym environments, along with the results of the flickering environments.

\subsection{Non-Markovian information} \label{subsec:qualities}

We experiment with other levels of information in the Varying Mountain Hike environments.
More precisely, we consider an information $i$ that contains an observation $\tilde x$ of the position $x$ (or an observation $\tilde y$ of the altitude $y$) with Gaussian noise of standard deviation $\sigma_i \in [0, \sigma_o]$.
In addition, in the case of environments with random orientation, we consider an information that also contains a noisy observation of the orientation $c$ replaced with a random orientation with probability $\epsilon_i \in [0, 1]$.
Note that when $\sigma_i = 0$, the exact position $x$ (or altitude $y$) is encoded in the information, while when $\sigma_i = \sigma_o$, the observation $o$ is encoded in the information.

\begin{figure}[ht]
    \centering
    \vspace{-2ex}
    \includegraphics[width=0.75\linewidth]{{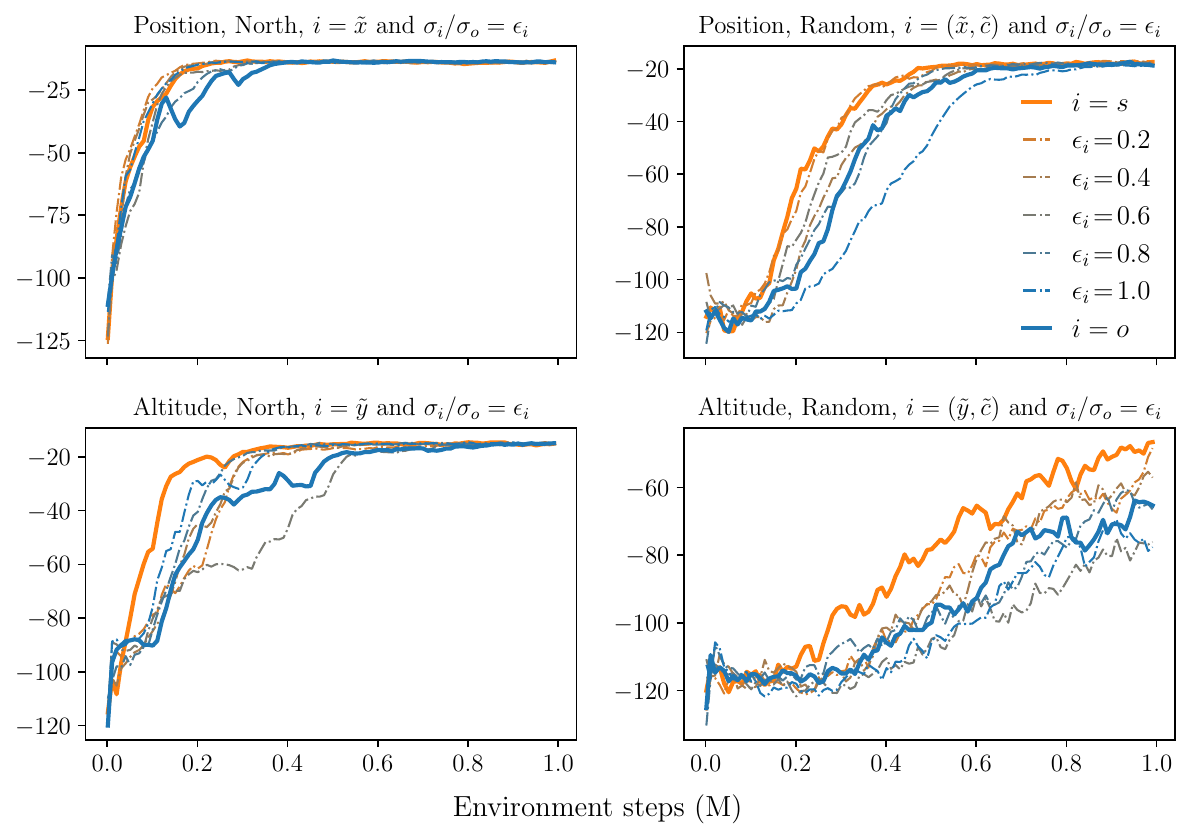}}
    \vspace{-3ex}
    \caption{%
        Varying Mountain Hike environments: average return of the Informed Dreamer with various level of information over five trainings.
    }
    \label{fig:qualities}
    \vspace{-0.5ex}
\end{figure}

As shown in \autoref{fig:qualities}, without confidence intervals for the sake of readability, the better the information, the faster the policy converges.
These results hold in all environments except that with altitude observation and fixed orientation, for which the results are more mixed.
As said in \autoref{subsec:hike}, it supports the hypothesis that the more informative about the state the information is, the faster an optimal policy is learned.
Moreover, it can be observed that when an additional information $\tilde c$ is not informative about the state, convergence is slower than for the Uninformed Dreamer.
This highlights again the importance of the quality of the additional information.

\subsection{Harder Pop Gym environments} \label{subsec:hard}

Despite the performance of the informed policy being equal to the performance of the uninformed policy at optimum, there may exists environments for which the optimum is never reached in practice without considering additional information at training time.
We observe it to be the case for environments with long time dependencies, such as the Repeat Previous environment of the Pop Gym suite.
In this subsection, we study in depth this failure case of the Uninformed Dreamer for this particular environment.
In the Repeat Previous environment, the agent is observing random noise, and is rewarded for outputting the observation that it got $k$ time steps ago.
While in \autoref{subsec:pop} we only considered the default Pop Gym environments, where $k = 4$ for the Repeat Previous environment, we here consider the Medium ($k = 32$) and Hard ($k = 64$) versions of this environment.

\begin{figure}[ht]
    \centering
    \vspace{-2ex}
    \includegraphics[width=0.65\linewidth]{{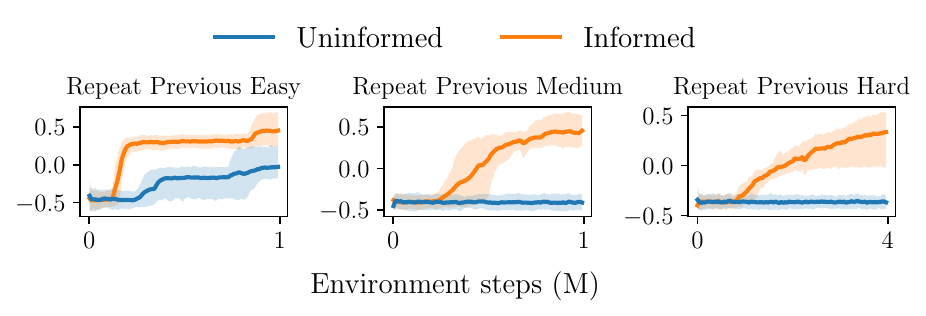}}
    \vspace{-2ex}
    \caption{
        Uninformed Dreamer and Informed Dreamer with $i = s$ in the Repeat Previous environments: minimum, maximum and average returns over five trainings.
    }
    \label{fig:hard}
    \vspace{-0.5ex}
\end{figure}

In \autoref{fig:hard}, we see that the Uninformed Dreamer is not able to improve the performance of its policy at all in these harder environments, while the Informed Dreamer still seems to converge towards a near-optimal policy.
It once again validates empirically the assumption that exploiting additional information about the state improves the speed of convergence towards an optimal policy.
Even more, it shows that exploiting additional information about the state can lead to convergence in environments where traditional approaches fail, such as those with long time dependencies.
The additional supervision provided by this Markovian information (the last $k$ observations) certainly endows the statistic $z \sim f(\cdot | h)$ with a useful encoding of the last $k$ observations, which is then decoded by the policy.
\autoref{tab:hard} provides the final return obtained by the Uninformed Dreamer and the Informed Dreamer for these environments.

\begin{table}[ht]
    \fontsize{9pt}{10pt}\selectfont
    \vspace{-0.5ex}
    \caption{
        Average final return and standard deviation over five trainings in the Repeat Previous environments.
    }
    \vspace{-1.5ex}
    \label{tab:hard}
    \begin{center}
        \input{tex/hard.tex}
    \end{center}
    \vspace{-2ex}
\end{table}

\subsection{Flickering Atari} \label{subsec:atari}

In the Flickering Atari environments, the agent is tasked with playing the Atari games \citep{bellemare2013arcade} on a flickering screen.
The dynamics are left unchanged, but the agent may randomly observe a blank screen instead of the game screen, with probability $p = 0.5$.
While the classic Atari games are known to have low stochasticity and few partial observability challenges \citep{hausknecht2015deep}, their flickering counterparts have constituted a classic benchmark in the partially observable RL literature \citep{hausknecht2015deep, zhu2017improving, igl2018deep, ma2020discriminative}.
Moreover, regarding the recent advances in sample-efficiency of model-based RL approaches, we consider the Atari 100k benchmark, where only 100k actions can be taken by the agent for generating samples of interaction.

For these environments, we consider the RAM state of the simulator, a $128$-dimensional byte vector, to be available as additional information for supervision.
This information vector is indeed guaranteed to satisfy the conditional independence of the informed POMDP: $p(o | i, s) = p(o | i)$.
Moreover, we postprocess this additional information by only selecting the subset of variables that are relevant to the game that is considered, according to the annotations provided by \citet{anand2019unsupervised}.
Depending on the game, this information vector might contain the number of remaining opponents, their positions, the player position, etc.

\begin{figure}[ht]
    \centering
    \vspace{-1ex}
    \includegraphics[width=\linewidth]{{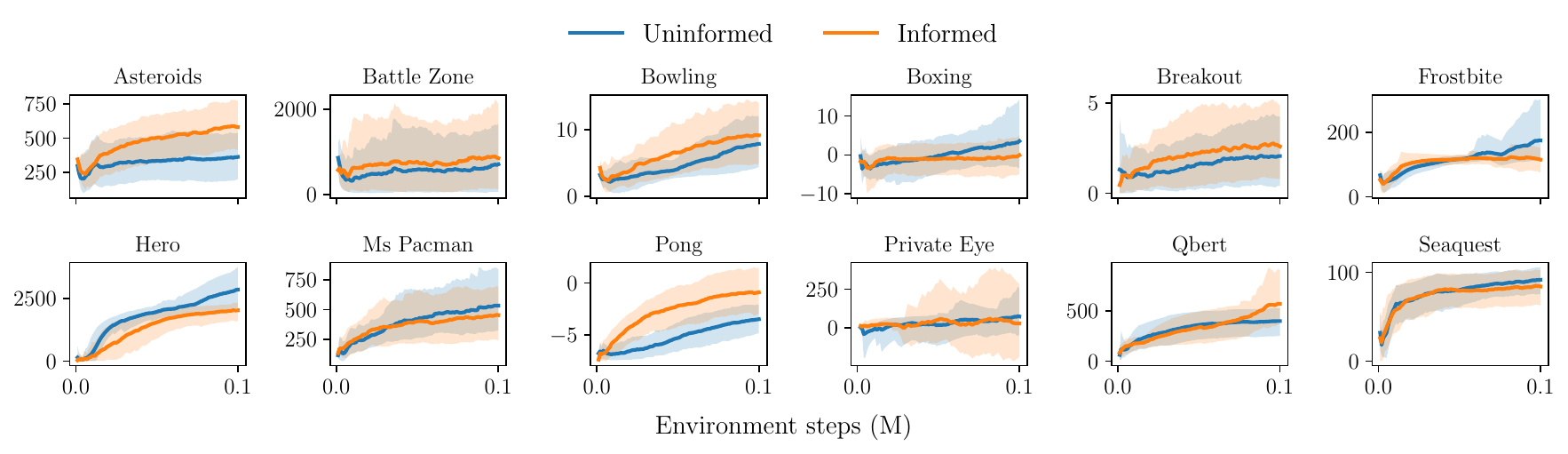}}
    \vspace{-3ex}
    \caption{
        Uninformed Dreamer and Informed Dreamer with $i = \phi(\text{RAM})$ in the Flickering Atari environments: minimum, maximum and average returns over five trainings.
    }
    \label{fig:atari}
    \vspace{-0.5ex}
\end{figure}

\autoref{fig:atari} shows that the speed of convergence and the performance of the policies is greatly improved by considering additional information for six environments, while degraded for four others and left similar for the rest.
The final returns are given in \autoref{tab:atari}, offering similar conclusions.

\begin{table}[ht]
    \fontsize{9pt}{10pt}\selectfont
    \vspace{-0.5ex}
    \caption{
        Average final return and standard deviation over five trainings in the Flickering Atari environments.
    }
    \vspace{-1.5ex}
    \label{tab:atari}
    \begin{center}
        \input{tex/atari.tex}
    \end{center}
    \vspace{-2ex}
\end{table}

\subsection{Flickering Control} \label{subsec:vision}

In the Flickering Control environments, the agent performs one of the standard DeepMind Control tasks from images but through a flickering screen.
As with the Flickering Atari environments, the dynamics are left unchanged, except that the agent may randomly observe a blank screen instead of the task screen, with probability $p = 0.5$.
For these environments, we consider the state to be available as additional information, as for the Velocity Control environments.

\begin{figure}[!ht]
    \centering
    \vspace{-1ex}
    \includegraphics[width=\linewidth]{{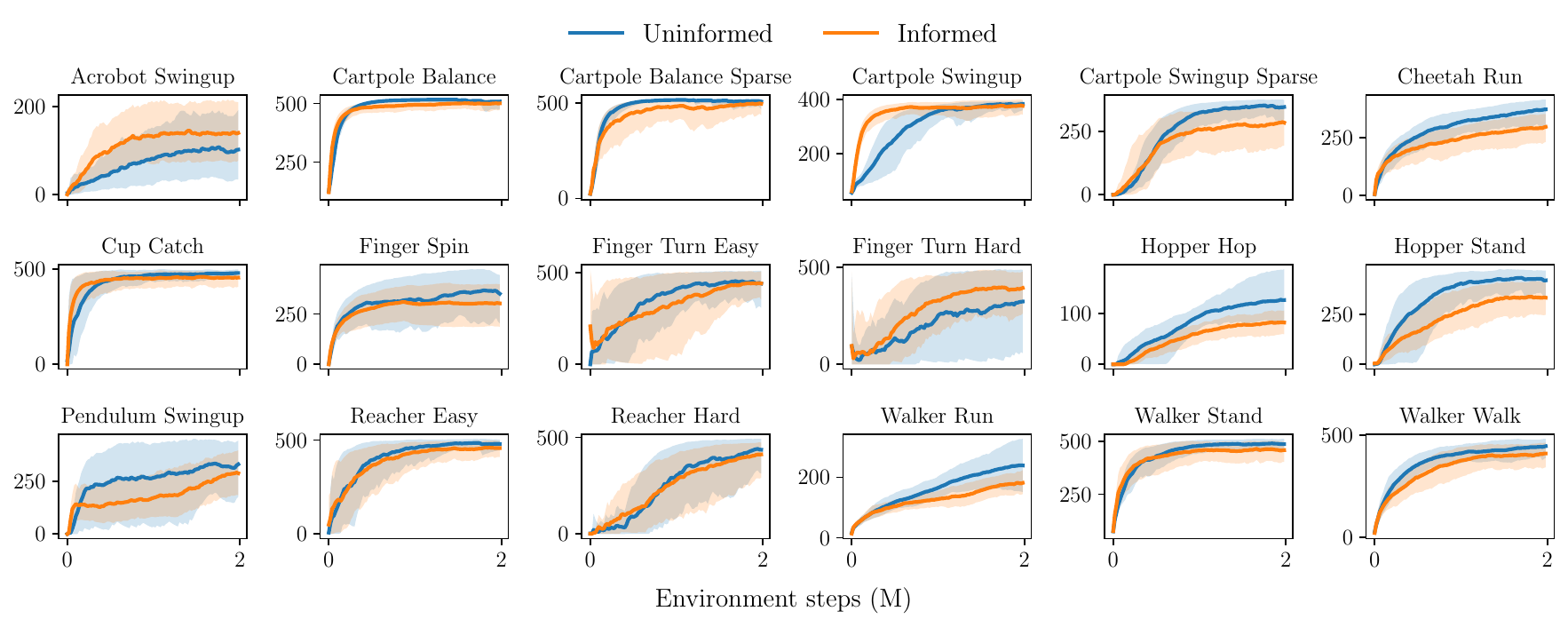}}
    \vspace{-3ex}
    \caption{
        Uninformed Dreamer and Informed Dreamer with $i = s$ in the Flickering Control environments: minimum, maximum and average returns over five trainings.
    }
    \label{fig:vision}
    \vspace{-0.5ex}
\end{figure}

Regarding this benchmark, considering additional information seems to degrade learning, generally resulting in worse policies.
This suggests that not all information is good to learn, some might be irrelevant to the control task and hinders the learning of optimal policies.
The final returns are given in \autoref{tab:vision}, and offer similar conclusions.

\begin{table}[!ht]
    \fontsize{9pt}{10pt}\selectfont
    \vspace{-0.5ex}
    \caption{
        Average final return and standard deviation over five trainings in the Flickering Control environments.
    }
    \vspace{-1.5ex}
    \label{tab:vision}
    \begin{center}
        \input{tex/vision.tex}
    \end{center}
    \vspace{-2ex}
\end{table}

\end{document}

%% file: tikz/styles.tex
\definecolor{MyYellow}{HTML}{F7D254}
\definecolor{MyBlue}{HTML}{3B7CF3}
\definecolor{MyRed}{HTML}{E8464A}
\definecolor{MyGreen}{HTML}{A3D16E}
\definecolor{MyLightGreen}{HTML}{729C44}
\definecolor{MyOrange}{HTML}{EF8C56}
\definecolor{MyPink}{HTML}{8C33B6}
\definecolor{MyPurple}{HTML}{982ABC}
\definecolor{MyGray}{HTML}{999999}
\definecolor{MyLightGray}{HTML}{616161}
\definecolor{MyLighterGray}{HTML}{909090}
\definecolor{MyLightBlue}{HTML}{52D6FC}
\definecolor{MyLighterBlue}{HTML}{81DEFE}

\tikzstyle{state} = [circle, fill=MyLightGray, text=white, text width=0.0cm, align=center]
\tikzstyle{reward} = [circle, fill=MyGreen, text=white, text width=0.0cm, align=center]
\tikzstyle{information} = [circle, fill=MyYellow, text=white, text width=0.0cm, align=center]
\tikzstyle{observation} = [circle, fill=MyRed, text=white, text width=0.0cm, align=center]
\tikzstyle{action} = [circle, fill=MyBlue, text=white, text width=0.0cm, align=center]
\tikzstyle{statistic} = [circle, fill=MyPurple, text=white, text width=0.0cm, align=center]
\tikzstyle{encoded} = [circle, fill=MyLightGray, text=white, text width=0.0cm, align=center]
\tikzstyle{prior} = [circle, fill=MyLighterGray, text=white, text width=0.0cm, align=center]
\tikzstyle{value} = [circle, fill=MyLightGreen, text=white, text width=0.0cm, align=center]

\tikzstyle{label} = [text=white, text width=0.2cm, align=center]
\tikzstyle{distribution} = [text=MyGray, text width=0.2cm, align=center]

\tikzstyle{arrow} = [-{Stealth[round]}, line width=0.020cm, color=MyGray]
\tikzstyle{loss} = [arrow, MyOrange]
\tikzstyle{bended} = [bend left=33]
\tikzstyle{zone} = [draw=none, rounded corners, fill=MyLightBlue]

%% file: tikz/ipomdp.tex
\draw[zone, opacity=0.30] (-0.4, -0.70) rectangle (4.4, -2.35) ;
\draw[zone, opacity=0.40] (-0.3, -1.75) rectangle (4.3, -2.25) ;

\input{tikz/trajectory}

\node[state] (s0) at (0, 0) {} ; \node[label] (s0t) at (s0) {$s$} ;
\node[state] (s1) at (2, 0) {} ; \node[label] (s1t) at (s1) {$s$} ;
\node[state] (s2) at (4, 0) {} ; \node[label] (s2t) at (s2) {$s$} ;

\draw[arrow] (s0) to (i0) ;
\draw[arrow] (i0) to (o0) ;

\draw[arrow] (s0) to (r0) ; \draw[arrow] (a0) to (r0) ;
\draw[arrow] (s0) to (s1) ; \draw[arrow] (a0) to (s1) ;

\draw[arrow] (s1) to (i1) ;
\draw[arrow] (i1) to (o1) ;

\draw[arrow] (s1) to (r1) ; \draw[arrow] (a1) to (r1) ;
\draw[arrow] (s1) to (s2) ; \draw[arrow] (a1) to (s2) ;

\draw[arrow] (s2) to (i2) ;
\draw[arrow] (i2) to (o2) ;

%% file: tikz/trajectory.tex
\input{tikz/history}

\node[information] (i0) at (0, -1) {}; \node[label] (i0t) at (i0) {$i$} ;
\node[information] (i1) at (2, -1) {}; \node[label] (i1t) at (i1) {$i$} ;
\node[information] (i2) at (4, -1) {}; \node[label] (i2t) at (i2) {$i$} ;

\node[reward] (r0) at (1, -1) {}; \node[label] (r0t) at (r0) {$r$} ;
\node[reward] (r1) at (3, -1) {}; \node[label] (r1t) at (r1) {$r$} ;

%% file: tikz/history.tex
\node[observation] (o0) at (0, -2) {}; \node[label] (o0t) at (o0) {$o$} ;
\node[observation] (o1) at (2, -2) {}; \node[label] (o1t) at (o1) {$o$} ;
\node[observation] (o2) at (4, -2) {}; \node[label] (o2t) at (o2) {$o$} ;

\node[action] (a0) at (1, -2) {}; \node[label] (a0t) at (a0) {$a$} ;
\node[action] (a1) at (3, -2) {}; \node[label] (a1t) at (a1) {$a$} ;

%% file: tikz/dynamics.tex
\node[distribution, above left=-0.15cm of s0] {$P$} ;
\node[distribution, above left=-0.15cm of i0] {$\tilde{I}$} ;
\node[distribution, above left=-0.15cm of o0] {$\tilde{O}$} ;
\node[distribution, below left=-0.15cm of r0] {$R$} ;
\node[distribution, above left=-0.15cm of s1] {$T$} ;

%% file: tikz/legend.tex
\node[label, text=MyGray, align=center, text width=2cm, anchor=center] (dots) at (5.2, 0) {$\dots$} ;
\node[label, text=MyLighterBlue, align=center, text width=2cm, anchor=center] (train) at (5.2, -1) {training} ;
\node[label, text=MyLightBlue, align=center, text width=2cm, anchor=center] (exec) at (5.2, -2) {execution} ;


%% file: tikz/initial.tex
\node[statistic] (z-1) at (-2.0, -3.1) {} ; \node[label] (z-1t) at (z-1) {$/$} ;
\node[action] (a-1) at (-1, -2) {} ; \node[label] (a-1t) at (a-1) {$/$} ;

%% file: tikz/recurrence.tex
\node[statistic] (z0)  at ( 0.0, -3.1) {} ; \node[label] (z0t) at (z0) {$z$} ;
\node[statistic] (z1)  at ( 2.0, -3.1) {} ; \node[label] (z1t) at (z1) {$z$} ;
\node[statistic] (z2)  at ( 4.0, -3.1) {} ; \node[label] (z2t) at (z2) {$z$} ;

\node[distribution, below=-0.05cm of z0] {$u_\theta$} ;

%% file: tikz/prior.tex
\node[prior] (p-1)     at (-1.0, \pos) {} ; \node[label] (p-1t) at (p-1) {$\hat e$} ;
\node[prior] (p0)      at ( 1.0, \pos) {} ; \node[label] (p0t) at (p0) {$\hat e$} ;
\node[prior] (p1)      at ( 3.0, \pos) {} ; \node[label] (p1t) at (p1) {$\hat e$} ;

\draw[arrow] (z-1) to (p-1) ;
\draw[arrow] (z0) to (p0) ;
\draw[arrow] (z1) to (p1) ;

\draw[arrow] (a-1) to[bend left=\bend] (p-1) ;
\draw[arrow] (a0) to[bend left=\bend] (p0) ;
\draw[arrow] (a1) to[bend left=\bend] (p1) ;

%% file: tikz/encoder.tex
\node[encoded] (e-1)   at (-1.0, -3.1) {} ; \node[label] (e-1t) at (e-1) {$e$} ;
\node[encoded] (e0)    at ( 1.0, -3.1) {} ; \node[label] (e0t) at (e0) {$e$} ;
\node[encoded] (e1)    at ( 3.0, -3.1) {} ; \node[label] (e1t) at (e1) {$e$} ;

\draw[arrow] (z-1) to (e-1) ; \draw[arrow] (a-1) to (e-1) ; \draw[arrow] (o0) to (e-1) ;
\draw[arrow] (z0) to (e0) ; \draw[arrow] (a0) to (e0) ; \draw[arrow] (o1) to (e0) ;
\draw[arrow] (z1) to (e1) ; \draw[arrow] (a1) to (e1) ; \draw[arrow] (o2) to (e1) ;

\draw[arrow] (z-1) to[bend right=25] (z0) ; \draw[arrow] (e-1) to (z0) ; \draw[arrow] (a-1) to (z0) ;
\draw[arrow] (z0) to[bend right=25] (z1) ; \draw[arrow] (e0) to (z1) ; \draw[arrow] (a0) to (z1) ;
\draw[arrow] (z1) to[bend right=25] (z2) ; \draw[arrow] (e1) to (z2) ; \draw[arrow] (a1) to (z2) ;

\node[distribution, above left=-0.10cm of e-1] {$q_\theta^e$} ;

%% file: tikz/loss.tex
\node[reward] (r-1) at (-1, -1) {} ; \node[label] (r-1t) at (r-1) {$/$} ;

\draw[loss] (e-1) to[bended] node[pos=0.75, left, align=right] {$q_\theta^r$} (r-1) ;
\draw[loss] (e0) to[bended] (r0) ;
\draw[loss] (e1) to[bended] (r1) ;

\draw[loss] (e-1) to node[pos=0.75, left, align=right] {$q_\theta^i$} (i0) ;
\draw[loss] (e0) to (i1) ;
\draw[loss] (e1) to (i2) ;

\draw[loss] (e-1) to node[pos=0.50, left=-0.1cm, align=right] {\tiny $- \operatorname{KL}$} (p-1) ;
\draw[loss] (e0) to (p0) ;
\draw[loss] (e1) to (p1) ;

%% file: tikz/policy.tex
\draw[arrow] (z0) to (a0) ;
\draw[arrow] (z1) to (a1) ;

\node[distribution, above left=-0.10 of a0] {$g_\phi$} ;

%% file: tikz/imagined.tex
\node[value] (v0) at (-0.4, -2) {}; \node[label] (v0t) at (v0) {$\hat v$} ;
\node[value] (v1) at (1.6, -2) {}; \node[label] (v1t) at (v1) {$\hat v$} ;
\node[value] (v2) at (3.6, -2) {}; \node[label] (v2t) at (v2) {$\hat v$} ;

\node[action] (a-1) at (-1, -2) {} ; \node[label] (a-1t) at (a-1) {$/$} ;
\node[action] (a0) at (1, -2) {} ; \node[label] (a0t) at (a0) {$a$} ;
\node[action] (a1) at (3, -2) {} ; \node[label] (a1t) at (a1) {$a$} ;

\node[reward] (r-1) at (-1.6, -2) {}; \node[label] (r-1t) at (r-1) {$/$} ;
\node[reward] (r0) at (0.4, -2) {}; \node[label] (r0t) at (r0) {$\hat r$} ;
\node[reward] (r1) at (2.4, -2) {}; \node[label] (r1t) at (r1) {$\hat r$} ;

\draw[loss] (p-1) to node[pos=0.75, left=-0.02cm, align=right] {$q_\theta^r$} (r-1) ; \draw[loss] (p-1) to node[pos=0.72, left=-0.1cm, align=right] {$v_\psi$} (v0) ;
\draw[loss] (p0) to (r0) ; \draw[loss] (p0) to (v1) ;
\draw[loss] (p1) to (r1) ; \draw[loss] (p1) to (v2) ;

\draw[arrow] (z-1) to[bend right=33] (z0) ; \draw[arrow] (p-1) to (z0) ; \draw[arrow] (a-1) to (z0) ;
\draw[arrow] (z0) to[bend right=33] (z1) ; \draw[arrow] (p0) to (z1) ; \draw[arrow] (a0) to (z1) ;
\draw[arrow] (z1) to[bend right=33] (z2) ; \draw[arrow] (p1) to (z2) ; \draw[arrow] (a1) to (z2) ;

%% file: tex/hike.tex
\begin{tabular}{cccc}
\toprule
Altitude & Random & Uninformed & Informed \\
\midrule
False & False & $-13.70 \pm 03.32$ & $\mathbf{-13.35 \pm 02.93}$ \\
False & True & $-18.32 \pm 06.04$ & $\mathbf{-17.72 \pm 04.19}$ \\
True & False & $\mathbf{-14.78 \pm 02.44}$ & $-14.98 \pm 04.73$ \\
True & True & $-67.05 \pm 21.76$ & $\mathbf{-45.94 \pm 32.77}$ \\
\bottomrule
\end{tabular}

%% file: tex/velocity.tex
\begin{tabular}{cccccccccccccccccc}
\toprule
Task & Uninformed & Informed \\
\midrule
Acrobot Swingup & $\mathbf{113.73 \pm 108.03}$ & $112.49 \pm 54.67$ \\
Cartpole Balance & $511.60 \pm 01.95$ & $\mathbf{513.22 \pm 00.82}$ \\
Cartpole Balance Sparse & $\mathbf{491.07 \pm 00.00}$ & $485.34 \pm 49.39$ \\
Cartpole Swingup & $347.58 \pm 18.30$ & $\mathbf{371.24 \pm 05.62}$ \\
Cartpole Swingup Sparse & $36.98 \pm 42.83$ & $\mathbf{102.44 \pm 139.79}$ \\
Cheetah Run & $\mathbf{315.40 \pm 39.64}$ & $305.91 \pm 103.62$ \\
Cup Catch & $465.23 \pm 28.77$ & $\mathbf{468.32 \pm 12.53}$ \\
Finger Spin & $186.66 \pm 39.34$ & $\mathbf{245.77 \pm 61.99}$ \\
Finger Turn Easy & $359.32 \pm 76.13$ & $\mathbf{414.82 \pm 46.09}$ \\
Finger Turn Hard & $347.91 \pm 81.80$ & $\mathbf{398.38 \pm 63.40}$ \\
Hopper Hop & $91.05 \pm 29.62$ & $\mathbf{97.50 \pm 29.83}$ \\
Hopper Stand & $350.77 \pm 88.92$ & $\mathbf{384.44 \pm 74.34}$ \\
Pendulum Swingup & $\mathbf{301.01 \pm 39.80}$ & $233.66 \pm 199.66$ \\
Reacher Easy & $463.30 \pm 17.78$ & $\mathbf{477.51 \pm 14.02}$ \\
Reacher Hard & $391.94 \pm 148.99$ & $\mathbf{466.35 \pm 25.94}$ \\
Walker Run & $238.07 \pm 76.42$ & $\mathbf{271.72 \pm 63.37}$ \\
Walker Stand & $\mathbf{462.81 \pm 18.20}$ & $460.51 \pm 41.87$ \\
Walker Walk & $429.65 \pm 27.06$ & $\mathbf{440.85 \pm 49.87}$ \\
\bottomrule
\end{tabular}

%% file: tex/pop.tex
\begin{tabular}{cccccccccc}
\toprule
Task & Uninformed & Informed \\
\midrule
Concentration & $\mathbf{00.01 \pm 00.16}$ & $-0.24 \pm 00.09$ \\
Count Recall & $-0.66 \pm 00.17$ & $\mathbf{-0.58 \pm 00.24}$ \\
Higher Lower & $\mathbf{00.39 \pm 00.07}$ & $00.31 \pm 00.12$ \\
Mine Sweeper & $\mathbf{-0.06 \pm 00.32}$ & $-0.07 \pm 00.38$ \\
Noisy Position Cart Pole & $00.21 \pm 00.19$ & $\mathbf{00.23 \pm 00.27}$ \\
Noisy Position Pendulum & $00.54 \pm 00.06$ & $\mathbf{00.55 \pm 00.05}$ \\
Position Cart Pole & $00.75 \pm 00.00$ & $\mathbf{00.75 \pm 00.00}$ \\
Position Pendulum & $00.64 \pm 00.07$ & $\mathbf{00.65 \pm 00.04}$ \\
Repeat First & $00.24 \pm 00.87$ & $\mathbf{00.56 \pm 01.00}$ \\
Repeat Previous & $-0.01 \pm 00.18$ & $\mathbf{00.44 \pm 00.13}$ \\
\bottomrule
\end{tabular}

%% file: tex/hard.tex
\begin{tabular}{ccc}
\toprule
Task & Uninformed & Informed \\
\midrule
Repeat Previous Easy & $-0.01 \pm 00.18$ & $\mathbf{00.44 \pm 00.13}$ \\
Repeat Previous Medium & $-0.41 \pm 00.06$ & $\mathbf{00.46 \pm 00.16}$ \\
Repeat Previous Hard & $-0.36 \pm 00.07$ & $\mathbf{00.33 \pm 00.19}$ \\
\bottomrule
\end{tabular}

%% file: tex/atari.tex
\begin{tabular}{cccccccccccc}
\toprule
Task & Uninformed & Informed \\
\midrule
Asteroids & $362.17 \pm 112.95$ & $\mathbf{580.92 \pm 95.61}$ \\
Battle Zone & $706.67 \pm 776.00$ & $\mathbf{849.61 \pm 357.35}$ \\
Bowling & $07.89 \pm 02.00$ & $\mathbf{09.17 \pm 01.24}$ \\
Boxing & $\mathbf{03.54 \pm 12.33}$ & $-0.06 \pm 05.66$ \\
Breakout & $02.06 \pm 01.32$ & $\mathbf{02.59 \pm 01.47}$ \\
Frostbite & $\mathbf{174.96 \pm 84.31}$ & $115.43 \pm 30.20$ \\
Hero & $\mathbf{2864.66 \pm 1054.84}$ & $2033.51 \pm 226.50$ \\
Ms Pacman & $\mathbf{534.67 \pm 117.97}$ & $455.02 \pm 155.17$ \\
Pong & $-3.49 \pm 01.19$ & $\mathbf{-0.90 \pm 01.78}$ \\
Private Eye & $\mathbf{74.27 \pm 42.00}$ & $29.66 \pm 67.47$ \\
Qbert & $401.27 \pm 117.26$ & $\mathbf{574.70 \pm 26.92}$ \\
Seaquest & $\mathbf{91.44 \pm 13.60}$ & $83.95 \pm 21.11$ \\
\bottomrule
\end{tabular}

%% file: tex/vision.tex
\begin{tabular}{cccccccccccccccccc}
\toprule
Task & Uninformed & Informed \\
\midrule
Acrobot Swingup & $104.87 \pm 54.88$ & $\mathbf{141.49 \pm 72.53}$ \\
Cartpole Balance & $\mathbf{508.01 \pm 00.92}$ & $499.95 \pm 24.87$ \\
Cartpole Balance Sparse & $\mathbf{507.94 \pm 03.04}$ & $495.14 \pm 69.63$ \\
Cartpole Swingup & $\mathbf{384.37 \pm 14.66}$ & $377.60 \pm 32.62$ \\
Cartpole Swingup Sparse & $\mathbf{347.07 \pm 27.63}$ & $284.53 \pm 72.05$ \\
Cheetah Run & $\mathbf{372.96 \pm 30.98}$ & $296.70 \pm 23.34$ \\
Cup Catch & $\mathbf{478.61 \pm 12.53}$ & $455.59 \pm 13.58$ \\
Finger Spin & $\mathbf{349.85 \pm 123.88}$ & $303.03 \pm 76.30$ \\
Finger Turn Easy & $\mathbf{441.53 \pm 47.13}$ & $441.16 \pm 66.91$ \\
Finger Turn Hard & $323.19 \pm 200.67$ & $\mathbf{392.48 \pm 85.25}$ \\
Hopper Hop & $\mathbf{126.72 \pm 37.89}$ & $81.92 \pm 19.90$ \\
Hopper Stand & $\mathbf{420.38 \pm 57.48}$ & $331.48 \pm 27.61$ \\
Pendulum Swingup & $\mathbf{329.35 \pm 82.31}$ & $286.53 \pm 102.18$ \\
Reacher Easy & $\mathbf{479.25 \pm 18.15}$ & $457.72 \pm 19.31$ \\
Reacher Hard & $\mathbf{433.40 \pm 214.42}$ & $412.97 \pm 27.10$ \\
Walker Run & $\mathbf{239.22 \pm 92.40}$ & $180.63 \pm 27.73$ \\
Walker Stand & $\mathbf{485.78 \pm 46.26}$ & $457.36 \pm 37.65$ \\
Walker Walk & $\mathbf{447.03 \pm 26.83}$ & $409.72 \pm 68.67$ \\
\bottomrule
\end{tabular}

%% file: main.bbl
\begin{thebibliography}{52}
\providecommand{\natexlab}[1]{#1}
\providecommand{\url}[1]{\texttt{#1}}
\expandafter\ifx\csname urlstyle\endcsname\relax
  \providecommand{\doi}[1]{doi: #1}\else
  \providecommand{\doi}{doi: \begingroup \urlstyle{rm}\Url}\fi

\bibitem[Anand et~al.(2019)Anand, Racah, Ozair, Bengio, C{\^o}t{\'e}, and
  Hjelm]{anand2019unsupervised}
Ankesh Anand, Evan Racah, Sherjil Ozair, Yoshua Bengio, Marc-Alexandre
  C{\^o}t{\'e}, and R~Devon Hjelm.
\newblock {Unsupervised State Representation Learning in Atari}.
\newblock \emph{Advances in Neural Information Processing Systems}, 32, 2019.

\bibitem[{\AA}str{\"o}m(1965)]{astrom1965optimal}
Karl~Johan {\AA}str{\"o}m.
\newblock {Optimal Control of Markov Processes with Incomplete State
  Information}.
\newblock \emph{Journal of Mathematical Analysis and Applications},
  10:\penalty0 174--205, 1965.

\bibitem[Avalos et~al.(2024)Avalos, Delgrange, Nowe, Perez, and
  Roijers]{avalos2024wasserstein}
Rapha{\"e}l Avalos, Florent Delgrange, Ann Nowe, Guillermo Perez, and
  Diederik~M Roijers.
\newblock {The Wasserstein Believer: Learning Belief Updates for Partially
  Observable Environments through Reliable Latent Space Models}.
\newblock In \emph{The Twelfth International Conference on Learning
  Representations}, 2024.

\bibitem[Baisero \& Amato(2022)Baisero and Amato]{baisero2022unbiased}
Andrea Baisero and Christopher Amato.
\newblock {Unbiased Asymmetric Reinforcement Learning under Partial
  Observability}.
\newblock In \emph{Proceedings of the 21st International Conference on
  Autonomous Agents and Multiagent Systems}, pp.\  44--52, 2022.

\bibitem[Baisero et~al.(2022)Baisero, Daley, and Amato]{baisero2022asymmetric}
Andrea Baisero, Brett Daley, and Christopher Amato.
\newblock {Asymmetric DQN for Partially Observable Reinforcement Learning}.
\newblock In \emph{Uncertainty in Artificial Intelligence}, pp.\  107--117.
  PMLR, 2022.

\bibitem[Bakker(2001)]{bakker2001reinforcement}
Bram Bakker.
\newblock {Reinforcement Learning with Long Short-Term Memory}.
\newblock \emph{Advances in Neural Information Processing Systems}, 14, 2001.

\bibitem[Bellemare et~al.(2013)Bellemare, Naddaf, Veness, and
  Bowling]{bellemare2013arcade}
Marc~G Bellemare, Yavar Naddaf, Joel Veness, and Michael Bowling.
\newblock {The Arcade Learning Environment: An Evaluation Platform for General
  Agents}.
\newblock \emph{Journal of Artificial Intelligence Research}, 47:\penalty0
  253--279, 2013.

\bibitem[Bernardo \& Smith(2009)Bernardo and Smith]{bernardo2009bayesian}
Jos{\'e}~M Bernardo and Adrian~FM Smith.
\newblock \emph{{Bayesian Theory}}, volume 405.
\newblock John Wiley \& Sons, 2009.

\bibitem[Buesing et~al.(2018)Buesing, Weber, Racaniere, Eslami, Rezende,
  Reichert, Viola, Besse, Gregor, Hassabis, et~al.]{buesing2018learning}
Lars Buesing, Theophane Weber, S{\'e}bastien Racaniere, SM~Eslami, Danilo
  Rezende, David~P Reichert, Fabio Viola, Frederic Besse, Karol Gregor, Demis
  Hassabis, et~al.
\newblock {Learning and Querying Fast Generative Models for Reinforcement
  Learning}.
\newblock \emph{arXiv preprint arXiv:1802.03006}, 2018.

\bibitem[Choudhury et~al.(2018)Choudhury, Bhardwaj, Arora, Kapoor, Ranade,
  Scherer, and Dey]{choudhury2018data}
Sanjiban Choudhury, Mohak Bhardwaj, Sankalp Arora, Ashish Kapoor, Gireeja
  Ranade, Sebastian Scherer, and Debadeepta Dey.
\newblock {Data-Driven Planning via Imitation Learning}.
\newblock \emph{The International Journal of Robotics Research}, 37\penalty0
  (13-14):\penalty0 1632--1672, 2018.

\bibitem[Chua et~al.(2018)Chua, Calandra, McAllister, and Levine]{chua2018deep}
Kurtland Chua, Roberto Calandra, Rowan McAllister, and Sergey Levine.
\newblock {Deep Reinforcement Learning in a Handful of Trials Using
  Probabilistic Dynamics Models}.
\newblock \emph{Advances in Neural Information Processing Systems}, 31, 2018.

\bibitem[Chung et~al.(2015)Chung, Kastner, Dinh, Goel, Courville, and
  Bengio]{chung2015recurrent}
Junyoung Chung, Kyle Kastner, Laurent Dinh, Kratarth Goel, Aaron~C Courville,
  and Yoshua Bengio.
\newblock {A Recurrent Latent Variable Model for Sequential Data}.
\newblock \emph{Advances in Neural Information Processing Systems}, 28, 2015.

\bibitem[Foerster et~al.(2018)Foerster, Farquhar, Afouras, Nardelli, and
  Whiteson]{foerster2018counterfactual}
Jakob Foerster, Gregory Farquhar, Triantafyllos Afouras, Nantas Nardelli, and
  Shimon Whiteson.
\newblock {Counterfactual Multi-Agent Policy Gradients}.
\newblock In \emph{Proceedings of the AAAI Conference on Artificial
  Intelligence}, volume 32 (1), 2018.

\bibitem[Gregor et~al.(2019)Gregor, Jimenez~Rezende, Besse, Wu, Merzic, and
  van~den Oord]{gregor2019shaping}
Karol Gregor, Danilo Jimenez~Rezende, Frederic Besse, Yan Wu, Hamza Merzic, and
  Aaron van~den Oord.
\newblock {Shaping Belief States with Generative Environment Models for RL}.
\newblock \emph{Advances in Neural Information Processing Systems}, 32, 2019.

\bibitem[Guo et~al.(2018)Guo, Azar, Piot, Pires, and Munos]{guo2018neural}
Zhaohan~Daniel Guo, Mohammad~Gheshlaghi Azar, Bilal Piot, Bernardo~A Pires, and
  R{\'e}mi Munos.
\newblock {Neural Predictive Belief Representations}.
\newblock \emph{arXiv preprint arXiv:1811.06407}, 2018.

\bibitem[Guo et~al.(2020)Guo, Pires, Piot, Grill, Altch{\'e}, Munos, and
  Azar]{guo2020bootstrap}
Zhaohan~Daniel Guo, Bernardo~Avila Pires, Bilal Piot, Jean-Bastien Grill,
  Florent Altch{\'e}, R{\'e}mi Munos, and Mohammad~Gheshlaghi Azar.
\newblock {Bootstrap Latent-Predictive Representations for Multitask
  Reinforcement Learning}.
\newblock In \emph{International Conference on Machine Learning}, pp.\
  3875--3886. PMLR, 2020.

\bibitem[Ha \& Schmidhuber(2018)Ha and Schmidhuber]{ha2018recurrent}
David Ha and J{\"u}rgen Schmidhuber.
\newblock {Recurrent World Models Facilitate Policy Evolution}.
\newblock \emph{Advances in Neural Information Processing Systems}, 31, 2018.

\bibitem[Hafner et~al.(2019)Hafner, Lillicrap, Fischer, Villegas, Ha, Lee, and
  Davidson]{hafner2019learning}
Danijar Hafner, Timothy Lillicrap, Ian Fischer, Ruben Villegas, David Ha,
  Honglak Lee, and James Davidson.
\newblock {Learning Latent Dynamics for Planning from Pixels}.
\newblock In \emph{{International Conference on Machine Learning}}, pp.\
  2555--2565. PMLR, 2019.

\bibitem[Hafner et~al.(2020)Hafner, Lillicrap, Ba, and
  Norouzi]{hafner2020dream}
Danijar Hafner, Timothy Lillicrap, Jimmy Ba, and Mohammad Norouzi.
\newblock {Dream to Control: Learning Behaviors by Latent Imagination}.
\newblock In \emph{International Conference on Learning Representations}, 2020.

\bibitem[Hafner et~al.(2021)Hafner, Lillicrap, Norouzi, and
  Ba]{hafner2021mastering}
Danijar Hafner, Timothy Lillicrap, Mohammad Norouzi, and Jimmy Ba.
\newblock {Mastering Atari with Discrete World Models}.
\newblock In \emph{International Conference on Learning Representations}, 2021.

\bibitem[Hafner et~al.(2023)Hafner, Pasukonis, Ba, and
  Lillicrap]{hafner2023mastering}
Danijar Hafner, Jurgis Pasukonis, Jimmy Ba, and Timothy Lillicrap.
\newblock {Mastering Diverse Domains through World Models}.
\newblock \emph{arXiv preprint arXiv:2301.04104}, 2023.

\bibitem[Han et~al.(2019)Han, Doya, and Tani]{han2019variational}
Dongqi Han, Kenji Doya, and Jun Tani.
\newblock {Variational Recurrent Models for Solving Partially Observable
  Control Tasks}.
\newblock In \emph{Internal Conference on Learning Representations}, 2019.

\bibitem[Hausknecht \& Stone(2015)Hausknecht and Stone]{hausknecht2015deep}
Matthew Hausknecht and Peter Stone.
\newblock {Deep Recurrent Q-Learning for Partially Observable MDPs}.
\newblock In \emph{2015 AAAI Fall Symposium Series}, 2015.

\bibitem[Heess et~al.(2015)Heess, Hunt, Lillicrap, and Silver]{heess2015memory}
Nicolas Heess, Jonathan~J Hunt, Timothy~P Lillicrap, and David Silver.
\newblock {Memory-Based Control with Recurrent Neural Networks}.
\newblock \emph{arXiv preprint arXiv:1512.04455}, 2015.

\bibitem[Hennig et~al.(2023)Hennig, Romero~Pinto, Yamaguchi, Linderman, Uchida,
  and Gershman]{hennig2023emergence}
Jay~A Hennig, Sandra~A Romero~Pinto, Takahiro Yamaguchi, Scott~W Linderman,
  Naoshige Uchida, and Samuel~J Gershman.
\newblock {Emergence of Belief-Like Representations through Reinforcement
  Learning}.
\newblock \emph{PLOS Computational Biology}, 19\penalty0 (9):\penalty0
  e1011067, 2023.

\bibitem[Hong et~al.(2022)Hong, Jin, and Tang]{hong2022rethinking}
Yitian Hong, Yaochu Jin, and Yang Tang.
\newblock {Rethinking Individual Global Max in Cooperative Multi-Agent
  Reinforcement Learning}.
\newblock \emph{Advances in Neural Information Processing Systems},
  35:\penalty0 32438--32449, 2022.

\bibitem[Igl et~al.(2018)Igl, Zintgraf, Le, Wood, and Whiteson]{igl2018deep}
Maximilian Igl, Luisa Zintgraf, Tuan~Anh Le, Frank Wood, and Shimon Whiteson.
\newblock {Deep Variational Reinforcement Learning for POMDPs}.
\newblock In \emph{International Conference on Machine Learning}, pp.\
  2117--2126. PMLR, 2018.

\bibitem[Lambrechts et~al.(2022)Lambrechts, Bolland, and
  Ernst]{lambrechts2022recurrent}
Gaspard Lambrechts, Adrien Bolland, and Damien Ernst.
\newblock {Recurrent Networks, Hidden States and Beliefs in Partially
  Observable Environments}.
\newblock \emph{Transactions on Machine Learning Research}, 2022.
\newblock ISSN 2835-8856.

\bibitem[Lee et~al.(2020)Lee, Nagabandi, Abbeel, and Levine]{lee2020stochastic}
Alex~X Lee, Anusha Nagabandi, Pieter Abbeel, and Sergey Levine.
\newblock {Stochastic Latent Actor-Critic: Deep Reinforcement Learning with a
  Latent Variable Model}.
\newblock \emph{Advances in Neural Information Processing Systems},
  33:\penalty0 741--752, 2020.

\bibitem[Leroy et~al.(2021)Leroy, Ernst, Geurts, Louppe, Pisane, and
  Sabatelli]{leroy2021qvmix}
Pascal Leroy, Damien Ernst, Pierre Geurts, Gilles Louppe, Jonathan Pisane, and
  Matthia Sabatelli.
\newblock {QVMix and QVMix-Max: Extending the Deep Quality-Value Family of
  Algorithms to Cooperative Multi-Agent Reinforcement Learning}.
\newblock In \emph{AAAI Workshop on Reinforcement Learning in Games}, 2021.

\bibitem[Li et~al.(2019)Li, Wu, Cui, Dong, Fang, and Russell]{li2019robust}
Shihui Li, Yi~Wu, Xinyue Cui, Honghua Dong, Fei Fang, and Stuart Russell.
\newblock {Robust Multi-Agent Reinforcement Learning via Minimax Deep
  Deterministic Policy Gradient}.
\newblock In \emph{Proceedings of the AAAI Conference on Artificial
  Intelligence}, volume 33 (01), pp.\  4213--4220, 2019.

\bibitem[Lowe et~al.(2017)Lowe, Wu, Tamar, Harb, Pieter~Abbeel, and
  Mordatch]{lowe2017multi}
Ryan Lowe, Yi~I Wu, Aviv Tamar, Jean Harb, OpenAI Pieter~Abbeel, and Igor
  Mordatch.
\newblock {Multi-Agent Actor-Critic for Mixed Cooperative-Competitive
  Environments}.
\newblock \emph{Advances in Neural Information Processing Systems}, 30, 2017.

\bibitem[Lyu et~al.(2022)Lyu, Baisero, Xiao, and Amato]{lyu2022deeper}
Xueguang Lyu, Andrea Baisero, Yuchen Xiao, and Christopher Amato.
\newblock {A Deeper Understanding of State-Based Critics in Multi-Agent
  Reinforcement Learning}.
\newblock In \emph{Proceedings of the AAAI Conference on Artificial
  Intelligence}, volume 36 (9), pp.\  9396--9404, 2022.

\bibitem[Ma et~al.(2020)Ma, Karkus, Hsu, Lee, and Ye]{ma2020discriminative}
Xiao Ma, Peter Karkus, David Hsu, Wee~Sun Lee, and Nan Ye.
\newblock {Discriminative Particle Filter Reinforcement Learning for Complex
  Partial Observations}.
\newblock In \emph{International Conference on Learning Representations}, 2020.

\bibitem[Morad et~al.(2023)Morad, Kortvelesy, Bettini, Liwicki, and
  Prorok]{morad2023popgym}
Steven Morad, Ryan Kortvelesy, Matteo Bettini, Stephan Liwicki, and Amanda
  Prorok.
\newblock {POPGym: Benchmarking Partially Observable Reinforcement Learning}.
\newblock In \emph{The Eleventh International Conference on Learning
  Representations}, 2023.

\bibitem[Nguyen et~al.(2021)Nguyen, Daley, Song, Amato, and
  Platt]{nguyen2021belief}
Hai Nguyen, Brett Daley, Xinchao Song, Christopher Amato, and Robert Platt.
\newblock {Belief-Grounded Networks for Accelerated Robot Learning under
  Partial Observability}.
\newblock In \emph{Conference on Robot Learning}, pp.\  1640--1653. PMLR, 2021.

\bibitem[Oliehoek et~al.(2008)Oliehoek, Spaan, and
  Vlassis]{oliehoek2008optimal}
Frans~A Oliehoek, Matthijs~TJ Spaan, and Nikos Vlassis.
\newblock {Optimal and Approximate Q-Value Functions for Decentralized POMDPs}.
\newblock \emph{Journal of Artificial Intelligence Research}, 32:\penalty0
  289--353, 2008.

\bibitem[Pinto et~al.(2018)Pinto, Andrychowicz, Welinder, Zaremba, and
  Abbeel]{pinto2017asymmetric}
Lerrel Pinto, Marcin Andrychowicz, Peter Welinder, Wojciech Zaremba, and Pieter
  Abbeel.
\newblock {Asymmetric Actor Critic for Image-Based Robot Learning}.
\newblock In \emph{14th Robotics: Science and Systems, RSS 2018}. MIT Press
  Journals, 2018.

\bibitem[Rashid et~al.(2018)Rashid, Samvelyan, Schroeder, Farquhar, Foerster,
  and Whiteson]{rashid2018qmix}
Tabish Rashid, Mikayel Samvelyan, Christian Schroeder, Gregory Farquhar, Jakob
  Foerster, and Shimon Whiteson.
\newblock {QMIX: Monotonic Value Function Factorisation for Deep Multi-Agent
  Reinforcement Learning}.
\newblock In \emph{International conference on machine learning}, pp.\
  4295--4304. PMLR, 2018.

\bibitem[Rashid et~al.(2020)Rashid, Farquhar, Peng, and
  Whiteson]{rashid2020weighted}
Tabish Rashid, Gregory Farquhar, Bei Peng, and Shimon Whiteson.
\newblock {Weighted QMix: Expanding Monotonic Value Function Factorisation for
  Deep Multi-Agent Reinforcement Learning}.
\newblock \emph{Advances in Neural Information Processing Systems},
  33:\penalty0 10199--10210, 2020.

\bibitem[Shao et~al.(2022)Shao, Kong, Matsumura, Fuji, Ito, and
  Mizuno]{shao2022mask}
Yang Shao, Quan Kong, Tadayuki Matsumura, Taiki Fuji, Kiyoto Ito, and Hiroyuki
  Mizuno.
\newblock {Mask Atari for Deep reinforcement Learning as POMDP Benchmarks}.
\newblock \emph{arXiv preprint arXiv:2203.16777}, 2022.

\bibitem[Son et~al.(2019)Son, Kim, Kang, Hostallero, and Yi]{son2019qtran}
Kyunghwan Son, Daewoo Kim, Wan~Ju Kang, David~Earl Hostallero, and Yung Yi.
\newblock {QTRAN: Learning to Factorize with Transformation for Cooperative
  Multi-Agent Reinforcement Learning}.
\newblock In \emph{International Conference on Machine Learning}, pp.\
  5887--5896. PMLR, 2019.

\bibitem[Subramanian et~al.(2022)Subramanian, Sinha, Seraj, and
  Mahajan]{subramanian2022approximate}
Jayakumar Subramanian, Amit Sinha, Raihan Seraj, and Aditya Mahajan.
\newblock {Approximate Information State for Approximate Planning and
  Reinforcement Learning in Partially Observed Systems}.
\newblock \emph{Journal of Machine Learning Research}, 23\penalty0
  (12):\penalty0 1--83, 2022.

\bibitem[Sutton(1991)]{sutton1991dyna}
Richard~S Sutton.
\newblock {Dyna, an Integrated Architecture for Learning, Planning, and
  Reacting}.
\newblock \emph{ACM Sigart Bulletin}, 2\penalty0 (4):\penalty0 160--163, 1991.

\bibitem[Tassa et~al.(2018)Tassa, Doron, Muldal, Erez, Li, Casas, Budden,
  Abdolmaleki, Merel, Lefrancq, et~al.]{tassa2018deepmind}
Yuval Tassa, Yotam Doron, Alistair Muldal, Tom Erez, Yazhe Li, Diego de~Las
  Casas, David Budden, Abbas Abdolmaleki, Josh Merel, Andrew Lefrancq, et~al.
\newblock {Deepmind Control Suite}.
\newblock \emph{arXiv preprint arXiv:1801.00690}, 2018.

\bibitem[Wang et~al.(2021)Wang, Ren, Liu, Yu, and Zhang]{wang2021qplex}
Jianhao Wang, Zhizhou Ren, Terry Liu, Yang Yu, and Chongjie Zhang.
\newblock {QPLEX: Duplex Dueling Multi-Agent Q-Learning}.
\newblock In \emph{International Conference on Learning Representations}, 2021.

\bibitem[Wang et~al.(2020)Wang, Everett, and How]{wang2020partially}
Rose~E Wang, Michael Everett, and Jonathan~P How.
\newblock {R-MADDPG for Partially Observable Environments and Limited
  Communication}.
\newblock \emph{arXiv preprint arXiv:2002.06684}, 2020.

\bibitem[Warrington et~al.(2021)Warrington, Lavington, Scibior, Schmidt, and
  Wood]{warrington2021robust}
Andrew Warrington, Jonathan~W Lavington, Adam Scibior, Mark Schmidt, and Frank
  Wood.
\newblock {Robust Asymmetric Learning in POMDPs}.
\newblock In \emph{International Conference on Machine Learning}, pp.\
  11013--11023. PMLR, 2021.

\bibitem[Wierstra et~al.(2010)Wierstra, F{\"o}rster, Peters, and
  Schmidhuber]{wierstra2010recurrent}
Daan Wierstra, Alexander F{\"o}rster, Jan Peters, and J{\"u}rgen Schmidhuber.
\newblock {Recurrent Policy Gradients}.
\newblock \emph{Logic Journal of the IGPL}, 18\penalty0 (5):\penalty0 620--634,
  2010.

\bibitem[Zhang et~al.(2016)Zhang, McCarthy, Finn, Levine, and
  Abbeel]{zhang2016learning}
Marvin Zhang, Zoe McCarthy, Chelsea Finn, Sergey Levine, and Pieter Abbeel.
\newblock {Learning Deep Neural Network Policies with Continuous Memory
  States}.
\newblock In \emph{IEEE International Conference on Robotics and Automation
  (ICRA)}, pp.\  520--527. IEEE, 2016.

\bibitem[Zhang et~al.(2019)Zhang, Vikram, Smith, Abbeel, Johnson, and
  Levine]{zhang2019solar}
Marvin Zhang, Sharad Vikram, Laura Smith, Pieter Abbeel, Matthew Johnson, and
  Sergey Levine.
\newblock {SOLAR: Deep Structured Representations for Model-Based Reinforcement
  Learning}.
\newblock In \emph{International Conference on Machine Learning}, pp.\
  7444--7453. PMLR, 2019.

\bibitem[Zhu et~al.(2017)Zhu, Li, Poupart, and Miao]{zhu2017improving}
Pengfei Zhu, Xin Li, Pascal Poupart, and Guanghui Miao.
\newblock {On Improving Deep Reinforcement Learning for POMDPs}.
\newblock \emph{arXiv preprint arXiv:1704.07978}, 2017.

\end{thebibliography}
